\documentclass[lettersize,journal]{IEEEtran}
\usepackage{amsmath,amsfonts}
\usepackage{algorithmic}
\usepackage{algorithm}
\usepackage{array}
\usepackage[caption=false,font=normalsize,labelfont=sf,textfont=sf]{subfig}
\usepackage{textcomp}
\usepackage{stfloats}
\usepackage{url}
\usepackage{verbatim}
\usepackage{graphicx}
\usepackage{cite}
\usepackage{amssymb}
\usepackage{bm}
\usepackage{hyperref}
\usepackage{color}
\usepackage{makecell}
\usepackage{multirow}
\usepackage{array}
\usepackage{booktabs}
\usepackage{arydshln}
\usepackage{enumitem}
\usepackage{cuted}
\usepackage{amsthm}

\newtheoremstyle{ieeenormal}
  {3pt}   % Space above
  {3pt}   % Space below
  {\normalfont} % Body font
  {}      % Indent amount
  {\bfseries} % Head font
  {.}     % Punctuation after head
  {0.5em} % Space after head
  {}

\theoremstyle{ieeenormal}
\newtheorem{theorem}{Theorem}
\newtheorem{lemma}{Lemma}
\newtheorem{proposition}{Proposition}

\newtheorem{remark}{Remark}
\theoremstyle{definition}

\begin{document}

% \newtheorem{theorem}{Theorem}[section]
% \newtheorem{proposition}[theorem]{Proposition}
% \newtheorem{lemma}[theorem]{Lemma}
% \newtheorem{corollary}[theorem]{Corollary}
% % \theoremstyle{definition}
% \newtheorem{definition}[theorem]{Definition}
% \newtheorem{assumption}[theorem]{Assumption}
% % \theoremstyle{remark}
% \newtheorem{remark}[theorem]{Remark}

\def\vzero{{\bm{0}}}
\def\vz{{\bm{z}}}
\def\valpha{{\bm{\alpha}}}
\def\vbeta{{\bm{\beta}}}
\def\vone{{\bm{1}}}
\def\vmu{{\bm{\mu}}}
\def\vtheta{{\bm{\theta}}}
\def\vsigma{{\bm{\sigma}}}
\def\vphi{{\bm{\phi}}}
\def\vf{{\bm{f}}}
\def\vg{{\bm{g}}}
\def\ve{{\bm{e}}}
\def\vu{{\bm{u}}}
\def\vw{{\bm{w}}}
\def\vx{{\bm{x}}}
\def\vy{{\bm{y}}}
\def\vs{{\bm{s}}}
\def\vc{{\bm{c}}}
\def\vq{{\bm{q}}}
\def\vb{{\bm{b}}}
\def\vxi{{\bm{\xi}}}
\def\veta{{\bm{\eta}}}
\def\vF{{\bm{F}}}
\def\vP{{\bm{P}}}
\def\vQ{{\bm{Q}}}
\def\vA{{\bm{A}}}
\def\vK{{\bm{K}}}
\def\vH{{\bm{H}}}
\def\vX{{\bm{X}}}
\def\vY{{\bm{Y}}}
\def\vG{{\bm{G}}}
\def\vM{{\bm{M}}}
\def\vE{{\bm{E}}}
\def\vI{{\bm{I}}}
\def\vB{{\bm{B}}}
\def\vS{{\bm{S}}}
\def\vU{{\bm{U}}}
\def\vZ{{\bm{Z}}}
\def\vW{{\bm{W}}}
\def\vL{{\bm{L}}}
\def\vJ{{\bm{J}}}
\def\vP{{\bm{P}}}
\def\vR{{\bm{R}}}
\def\vdelta{{\bm{\delta}}}
\def\md{{\mathrm{d}}}

\title{Lyapunov Guidance: A Unified Framework for Stabilizing Generative Flows}

\author{Jingdong Zhang, Xinze Li, Yize Jiang, Luan Yang, Minkai Xu, Junhong Liu
% IEEE Publication Technology,~\IEEEmembership{Staff,~IEEE,}
        % <-this % stops a space
    \thanks{J. Zhang is with Department of
Mathematics and I-X, Imperial College London, London, United Kingdom. E
mail: {j.zhang1}@imperial.ac.uk}
  \thanks{L. Yang is with Research Institute of Intelligent Complex Systems, Fudan University,  Shanghai, China. Email: luanyang23@m.fudan.edu.cn}
\thanks{M. Xu is with Department of Computer Science, Stanford University, Stanford, United States. Email: minkai@cs.stanford.edu}
\thanks{X. Li, Y. Jiang and J. Liu are with MicroCyto, Beijing, China. Email: {lixinze,jiangyize,liujunhong}@microcyto.cn}
% \thanks{Manuscript received ..; revised ..}
}



\maketitle

\begin{abstract}
Flow matching has emerged as an effective framework for learning complex data distributions, but adapting pretrained flow models to new tasks often requires computationally expensive retraining. Post-training guidance provides a more efficient alternative, but existing methods are largely heuristic and offer no explicit stability guarantees. We address this limitation by proposing LyaGuide, a unified Lyapunov-guided framework that formulates flow guidance as a Lyapunov control problem. Our main theoretical result establishes an equivalence between guided flow matching and Lyapunov control, thereby unifying common guidance strategies, such as classifier guidance, reward guidance, and energy-based guidance, within a single control-theoretic framework. To enforce the Lyapunov condition, we introduce a pseudo-projection operator with a closed-form expression that endows learned or heuristic guidance terms with explicit stability guarantees. LyaGuide supports two practical settings: a model-driven setting, where the target guidance distribution is specified through a known Lyapunov function, and a data-driven setting, where the guidance is adapted from task-specific downstream data. LyaGuide is compatible with existing guidance methods, introduces minimal additional computational overhead, and is straightforward to integrate in practice. Extensive experiments on synthetic benchmarks, image inverse problems, reinforcement learning planning, and energy-based modeling demonstrate consistent improvements in sample quality, guidance fidelity, and robustness, while maintaining computational efficiency.
\end{abstract}

\begin{IEEEkeywords}
Guided Generative Modeling, Flow Matching, Lyapunov Control, Pseudo-Projection
\end{IEEEkeywords}

\section{Introduction}
% \IEEEPARstart{T}{his}
{\color{black}
Generative modeling has achieved substantial progress with the development of diffusion models~\cite{song2021score,dhariwal2021diffusion,ho2021classifier} and their deterministic counterparts based on flow matching~\cite{lipman2023flow,liuflow,tong2023conditional}. By learning a time-dependent vector field that transports a simple base distribution to a complex data distribution, flow matching provides a mathematically principled and computationally efficient alternative to stochastic diffusion processes. Despite this success, adapting a pretrained flow model to downstream tasks remains challenging. A direct solution is to retrain or fine-tune the model using task-specific data, but such strategies are often computationally expensive and require task-dependent optimization. Post-training guidance offers a more flexible alternative by modifying the inference dynamics without retraining the underlying model. However, existing guidance methods are largely heuristic or approximation based~\cite{fan2025cfg,fengguidance}, and they usually lack explicit guarantees on the stability and reliability of the guided generation process. This limitation is particularly restrictive in applications where guidance must incorporate heterogeneous forms of prior information, such as class labels~\cite{dhariwal2021diffusion,ho2021classifier}, structural constraints in molecular design~\cite{zhang2024generalized}, and reward functions in reinforcement learning~\cite{janner2022planning}. These observations motivate the development of a unified theoretical framework for guided flow models that can accommodate diverse conditioning mechanisms while providing principled stability guarantees.

Although the aforementioned guidance scenarios arise in different application domains, they share a common objective: to modify the generative dynamics so that the resulting trajectories are biased toward regions that better satisfy task-specific preferences or constraints. In classifier guidance and classifier-free guidance, the additional signal is induced by label information~\cite{dhariwal2021diffusion,ho2021classifier}; in molecular and structural design, the guidance is often induced by domain-dependent energies or constraints rooted in physical and biological knowledge, such as salt bridges, hydrogen bonds, and hydrophobic interactions~\cite{zhang2024generalized}; in reinforcement learning and planning, it is often encoded through predefined reward functions or value-related objectives~\cite{janner2022planning}. Similar ideas also appear in image inverse problems, where the generation process is guided toward data-consistent solutions under measurement constraints~\cite{song2021solving,kawar2022denoising,chungdiffusion}. Despite their empirical success, these methods are usually developed in a task-specific manner, with guidance terms derived from heterogeneous objectives and implemented through separate heuristics. As a result, it remains unclear what common principle underlies these seemingly different strategies, when such guidance can be interpreted as a principled modification of the generative dynamics, and how one can impose reliable guarantees on the resulting guided process.

This viewpoint is closely related to a long-standing theme in control theory of dynamical systems, where the central objective is to steer temporal trajectories toward desired target states. In classical settings, Lyapunov stability theory provides a principled foundation for designing stabilizing policies for linear and polynomial systems, including linear quadratic regulation (LQR)~\cite{khalil2002nonlinear} and semidefinite-programming-based sum-of-squares (SOS) methods~\cite{parrilo2000structured}. For more intricate high-dimensional and nonlinear dynamics, recent advances have increasingly combined control theory with machine learning~\cite{tsukamoto2021contraction}. In particular, neural controllers have been developed together with certificate functions, such as Lyapunov functions, LaSalle invariants, barrier certificates, and contraction metrics~\cite{chang2020neural,zhang2022neural,yang2025neural,zhang2022sync,qin2020learning,sun2021learning}. These developments suggest that guided generation can be studied from a control-theoretic perspective, in which the guidance term is interpreted as a control input that reshapes the generative dynamics toward task-preferred regions, while certificate functions provide a principled mechanism for analyzing stability and convergence.

In this work, we develop a unified framework for flow guidance from the perspective of Lyapunov control theory~\cite{artstein1983stabilization,sontag1989universal,polyakov2012nonlinear}. Our central idea is to interpret the energy or potential function that defines the conditional distribution as a Lyapunov function, and to view the added guidance term as a control input that stabilizes the generative dynamics toward task-preferred regions. Based on this perspective, our main theoretical result in Theorem~\ref{main thm} establishes  an equivalence between guided flow matching and Lyapunov control, which in turn provides a common interpretation for a broad class of existing guidance strategies, including classifier guidance, energy-based guidance, and reward-guided generation, as concluded in Proposition~\ref{prop1} and illustrated in Fig.~\ref{fig sketch}. Beyond this conceptual unification, we further introduce a pseudo-projection operator that enforces the Lyapunov condition for a candidate guidance term through a closed-form correction, thereby providing a principled mechanism for stabilizing learned or heuristic guidance functions. At the same time, the pseudo-projection operator is easy to incorporate into existing flow guidance methods, requiring only a minimal modification to the implementation. In this way, our framework connects theoretical interpretability, rigorous stability analysis, and practical compatibility with existing guidance methods.

A practical challenge in this framework is how to obtain the Lyapunov function $V$ that encodes task-specific prior knowledge. In some applications, $V$ can be specified explicitly. Examples include classifier guidance, where $V(\vx)\propto-\log p(y|\vx)$ is induced by label information in image generation~\cite{dhariwal2021diffusion}, and molecular or structural design, where V may be constructed from domain-dependent energies or constraints derived from physical and biological knowledge~\cite{zhang2024generalized}. In many other applications, however, such prior knowledge is not analytically available and can only be inferred implicitly from limited downstream supervision~\cite{janner2022planning,lee2025calibrated,blacktraining}. To address both situations within a unified framework, we consider two practical settings: a model-driven setting, which is particularly suitable for developers who can explicitly specify the Lyapunov function from domain knowledge, and a data-driven setting, which is more suitable for user-side deployment where the Lyapunov function must be inferred from limited task-specific data and then used to construct guidance policies. This distinction allows LyaGuide to cover both analytically defined guidance objectives and scenarios in which the task prior must be estimated from data.

\textbf{Contributions.} The main contributions of this work are summarized as follows:
\begin{itemize}[itemsep=-0.0em,topsep=-0.45em,leftmargin=2em]
    \item We establish a unified theoretical framework that formulates guided flow matching as a Lyapunov control problem, and prove an equivalence between guided flow matching and Lyapunov control. This result provides a common control-theoretic interpretation of representative guidance settings, including classifier guidance, reward guidance, energy-based guidance, and image inverse problems.
    
    \item We introduce a pseudo-projection operator with a closed-form expression to enforce the Lyapunov condition for candidate guidance terms. This operator provides a principled mechanism for stabilizing learned or heuristic guidance functions, while remaining easy to integrate into existing flow-guidance methods.
    
    \item Building on the proposed theory, we develop \textbf{LyaGuide}, which supports both model-driven and data-driven settings, thereby covering scenarios with explicitly specified priors as well as scenarios where task-specific priors must be learned from limited downstream data. Extensive experiments on synthetic benchmarks, image inverse problems, reinforcement learning planning, and energy-based modeling demonstrate that LyaGuide consistently improves guidance quality, sample quality, and robustness while maintaining computational efficiency.
    
\end{itemize}

}

\section{Related Work}\label{sec related work}

\subsection{Guidance in Generative Models.}
Guidance has become a central mechanism in controllable generative modeling. In diffusion models, classifier guidance~\cite{dhariwal2021diffusion} and classifier-free guidance~\cite{ho2021classifier} are two of the most influential approaches, with subsequent extensions incorporating reward-based and energy-based priors~\cite{janner2022planning,zhang2023towards,lecun2006ebm}. For flow matching, recent work has studied guidance directly at the vector-field level and shown how importance reweighting on the joint distribution induces pathwise changes in the marginal dynamics~\cite{fengguidance}. More closely related to our perspective, Sprague et al.~\cite{sprague2024stability} introduced stochastic stability considerations into flow matching. Compared with these approaches, our framework provides a unified control-theoretic interpretation of diverse guidance strategies and introduces a pseudo-projection operator to enforce Lyapunov inequalities during inference. In addition, Albergo et al.~\cite{albergo2023stochastic} and Chen et al.~\cite{chen2023phase} proposed alternative generative formulations based on stochastic interpolants and phase-space bridges, respectively. While these methods also reveal connections between generative dynamics and control-inspired principles, they reformulate the generative process itself, whereas our goal is to provide a general-purpose guidance framework that can be directly applied to pretrained flow models.

\subsection{Learning Certificates.}
Learning certificate functions has become an important direction in the intersection of machine learning, control, and dynamical systems. In score-based modeling, the score $\nabla \log p(\vx)=-\nabla V(\vx)$ can be viewed as a learnable vector field induced by an implicit potential~\cite{song2019generative,song2021score}, while energy-based models directly parametrize an energy function~\cite{lecun2006ebm,du2019implicit}. In parallel, few-shot or preference-based priors have been explored in reinforcement learning and generative guidance~\cite{janner2022planning,lee2025calibrated}, but typically without explicit Lyapunov certificates. Beyond generative modeling, Liu et al.~\cite{liu2023pinn} studied physics-informed neural networks for learning and verifying Lyapunov functions in PDE systems, and Kang et al.~\cite{kang2021stable} developed stable neural ODEs with Lyapunov-stable equilibria. These works are conceptually related to our data-driven setting, where a Lyapunov candidate $V_\theta$ is learned from sparse supervision and then incorporated into the guidance design. Our framework differs in that it is tailored to generative modeling and combines learned certificate functions with an explicit projection-based stabilization mechanism.

\subsection{Diffusion Models and Flow Matching for Inverse Problems.}
Diffusion models and flow matching have been widely used in inverse problems as powerful generative priors. In diffusion-based formulations, a learned score model is combined with the likelihood term induced by the forward operator, enabling posterior sampling for image restoration, compressed sensing, and medical imaging tasks~\cite{song2021score,chungdiffusion,wangzero,song2023pseudoinverse}. Flow matching offers a deterministic counterpart that learns a continuous transport field through the continuity equation~\cite{lipman2023flow,tong2023conditional,liuflow}, and has recently been extended to inverse problems through likelihood-guided transport, pseudoinverse-based correction, and plug-and-play reconstruction schemes~\cite{pokle2023training,zhang2024flow,martin2025pnp,yan2025fig}. These works show that pretrained generative flows can be adapted to satisfy measurement constraints at inference time, often with improved efficiency relative to diffusion-based samplers.

While our framework is developed from a general perspective, it is also directly applicable to inverse-problem settings. In particular, LyaGuide can be combined in a plug-and-play manner with existing flow-based inverse-problem methods to improve their guidance quality and sampling robustness, while potentially accelerating the sampling process through Lyapunov-induced contraction. From this viewpoint, our framework also reveals an essential underlying principle behind these seemingly different techniques: they all seek an appropriate pair of a potential function $V$ and a guidance field $u$ so that the generative dynamics are steered toward the ground-truth distribution under measurement constraints that can be encoded in Lyapunov conditions.

\subsection{Meta-learning and Task-specific Adaptation Modules.}
% Meta-learning aims to enable models to generalize rapidly to new tasks with limited data or adaptation. Among existing approaches, a representative line of work avoids modifying the entire network and instead introduces lightweight task-adaptive modules to adjust inference while keeping the base model fixed. Typical examples include feature-wise transformations such as FiLM layers, task-conditioning adapters, and modular meta-learning architectures~\cite{finn2017model,rusu2018meta,rebuffi2017learning}. These methods are usually developed within a meta-training/meta-testing paradigm over a distribution of tasks. By contrast, our method focuses on post-training adaptation of generative flows to new conditioning objectives by introducing Lyapunov-guided correction terms without retraining or task-specific fine-tuning. While sharing the general goal of lightweight adaptation, our approach is distinguished by its control-theoretic formulation and by the explicit stability guarantees provided through the proposed projection mechanism.

Meta-learning aims to enable models to generalize rapidly to new tasks with limited data or adaptation. Among existing approaches, a representative line of work avoids modifying the entire network and instead introduces lightweight task-adaptive modules to adjust inference while keeping the base model fixed. Typical examples include feature-wise transformations and conditional modulation layers such as FiLM~\cite{perez2018film}, residual adapters and efficient multi-domain parameterizations~\cite{rebuffi2017learning,rebuffi2018efficient}, as well as modular or fast-adaptation meta-learning architectures~\cite{finn2017model,rusu2018meta}. Related ideas have also been developed in parameter-efficient transfer learning through adapter-based modules~\cite{houlsby2019parameter}. These methods are usually developed within a meta-training/meta-testing paradigm over a distribution of tasks. By contrast, our method focuses on post-training adaptation of generative flows to new conditioning objectives by introducing Lyapunov-guided correction terms without retraining or task-specific fine-tuning. While sharing the general goal of lightweight adaptation, our approach is distinguished by its control-theoretic formulation and by the explicit stability guarantees provided through the proposed projection mechanism.

% \subsection{Generative Modeling in Latent or Alternative Spaces.}
% Another related line of work modifies the space in which generative dynamics are defined. For example, Dockhorn et al.~\cite{dockhorn2022latent} proposed score-based generative modeling in latent spaces to improve efficiency and representation quality. These approaches are complementary to ours: rather than changing the domain of the generative dynamics, our framework provides a stabilizing principle for guidance that is applicable regardless of whether the flow is defined in the original space or in an alternative representation space.

\begin{figure}
% \vskip -0.17in
	\centering
	\includegraphics[width=0.4\textwidth]{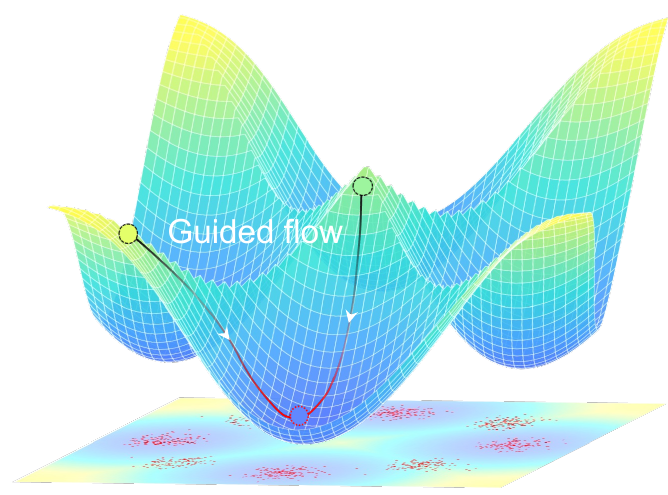}
	\caption{Lyapunov(energy) landscape for 8-Gaussian guidance from initial (black dashed) to target data (red dashed).
     }
     \label{fig sketch}
% \vskip -0.35in
\end{figure}

\section{Preliminaries}

\subsection{Flow Matching}
Flow matching~\cite{lipman2023flow} is an ODE-based training framework for generative modeling, which learns continuous-time dynamics that transforms a simple prior distribution to a complex data distribution.
Formally, let $p_0$ and $p_1$ denote the source and target distributions over $\mathbb{R}^d$, respectively. Flow matching aims to learn a time-dependent vector field $\vu(\vx,t)\triangleq \vu_t(\vx): \mathbb{R}^d \times [0,1] \to \mathbb{R}^d$ that pushes $p_0$ to $p_1$ along a continuous path $\{p_t\}_{t\in[0,1]}$, such that the flow $\phi_t$ governed by the ODE:
\begin{equation}
\frac{d\phi_t(\vx)}{dt} = \vu_t(\phi_t(\vx)), \quad \phi_0(\vx) = \vx,
\end{equation}
transfers $\vx \sim p_0$ to a sample $\phi_1(\vx) \sim p_1$.

To learn a neural field $\hat{\vu}_\vtheta(\vx,t)$ that matches the continuous path $\{p_t\}_{t\in[0,1]}$, a convenient way to define targets is via a latent variable $z \sim p(z)$ that indexes conditional bridges $p_t(\vx_t \mid z)$ together with conditional vector fields $\vu_{t\mid z}(\vx_t\mid z)$.
This induces the marginal path and vector field $
p_t(\vx_t)=\int p_t(\vx_t\mid z)\,p(z)\,dz$, 
$\vu_t(\vx_t)=\int \vu_{t\mid z}(\vx_t\mid z)\,p(z\mid \vx_t)\,dz$, and it is known that $\vu_t$ generates the marginal path $p_t$ (see~\cite{lipman2023flow}). To avoid the intractable term $p(z\mid \vx_t)$ in $\vu_t(\vx_t)$, conditional flow matching (CFM) proposes to train the model with an equivalent and tractable conditional objective~\cite{lipman2023flow,tong2023conditional}:
$
\mathcal{L}_{\text{cond}}(\vtheta)
=\mathbb{E}_{t\sim\mathcal{U}(0,1),\,z\sim p(z),\,\vx_t\sim p_t(\vx_t\mid z)}
\Big[\big\|\hat{\vu}_\theta(\vx_t,t)-\vu_{t\mid z}(\vx_t| z)\big\|_2^2\Big],$
whose minimizer coincides with that of conditional loss while remaining simulation-free and easy to estimate.

\subsection{Lyapunov Control Theory}
To begin with, we consider the feedback-controlled dynamic system of the following general form:
\begin{equation}\label{ODE0}
	\dot{\vx} = \vf_t({\vx},\vu(\vx)),~\vx\in\mathbb{R}^d,~\vu\in\mathbb{R}^m,
\end{equation}
where $\vf_t$ is the Lipschitz-continuous vector field acting on some prescribed open set $\vx\in\mathcal{D}\subset\mathbb{R}^d$. The solution initiated at time $t_0$ from $\vx_0$ is denoted by $\vx_t(t_0,\vx_0)$. We assume that the stationary target position of the controlled system is the origin, i.e. $\vf_t(\vzero,\vzero)=\vzero$. One major problem in control theory is to design stabilizing controller $\vu(\vx)$~\cite{wiener2019cybernetics} such that $\lim\limits_{t\to\infty}\vx_t(t_0,\vx_0)=\vzero$, for any initial value $\vx_0\in\mathcal{D}$. The following stability theory has been employed to devise the control policy that can stabilize the system when the equilibrium $\vzero$ is unstable.

% Based on the following Lyapunov stability theory, Chang et al \cite{chang2020neural} propose to simultaneously learn a neural controller $\vu$ and a neural Lyapunov function $V$.

\begin{theorem}\label{prop ODE}\cite{mao2007stochastic}
	Suppose that there exists a continuously differentiable function $V: \mathcal{D} \to R$ that satisfies the following conditions: $\mathrm{(i)}$
	$V (\vzero) = 0$, $\mathrm{(ii)}$ $V(\vx)\ge c\Vert\vx\Vert^p$ for some constants $c,p>0$, $\mathrm{(iii)}$ and $\mathcal{L}_{\vf}V < -\delta V$, for some $\delta>0$. \footnote{$\mathcal{L}_{\vf}V$ represent the Lie derivative of $V$ along the direction $\vf$, i.e., $\mathcal{L}_{\vf}V=\nabla V\cdot\vf_t$. }
	Then, the system is exponentially stable at the origin, that is, $\limsup_{t\to\infty}\frac{1}{t}\log\|\vx(t;t_0,\vx_0)\|\le -\frac{\delta}{p}$. Here $V$ is called a Lyapunov function.
\end{theorem}
For a given dynamic system~\eqref{ODE0}, the design of appropriate control policies that satisfies the Lyapunov condition in Theorem~\ref{prop ODE} has been a central topic~\cite{chang2020neural,zhang2022neural,dawson2023safe,yang2025neural}.

\textbf{Problem Statement.} 
We assume that flow matching has already learned a base vector field $\vu$ transferring the noise distribution $p_0$ to the data distribution $p_1$. 
For downstream tasks, this vector field must be adapted to generate task-specific conditional distributions $\frac{1}{Z}p_1(\vx)e^{-J(\vx)}$, which is achieved by introducing a Lyapunov function $V$ that relates to energy distribution $e^{-J}$ and encodes task-related priors, and thus the guided field $\vu + \vc$ \emph{rigorously} samples from the target distribution, where $\vc$ is a guidance control term to be designed such that it encodes the information of the Lyapunov function $V$, i.e., satisfies the Lyapunov condition.

There are mainly two scenarios depending on how prior knowledge is provided, which are described in detail as follows.

% \begin{myquote}
\textbf{Scenario 1 (Developers-Oriented).} Domain knowledge can be explicitly formulated as an analytical potential $V(\vx)$, making the conditioning objective transparent.
% \end{myquote}

Typical examples include classifier guidance with class labels~\cite{dhariwal2021diffusion,ho2021classifier}, structural constraints in protein design~\cite{zhang2024generalized}, and reward functions in reinforcement learning~\cite{janner2022planning}. 
Despite their different forms, we provide a unified framework to equate flow matching guidance in Scenario 1 with the Lyapunov control problem (see Section~\ref{sec sce 1}), and further introduce an efficient and training-free guidance policy grounded in this theoretical framework~(see Section~\ref{sec proj}).

% \begin{myquote}
\textbf{Scenario 2 (Users-Oriented).} Prior knowledge is implicit and can only be learned through few-shot learning with some data–score pairs $\{(\vx_i,V_i)\}_{i=1}^n$, without an explicit conditioner.
% \end{myquote}

In contrast to Scenario~1, the challenge here is to infer a Lyapunov function and corresponding control from sparse or noisy supervision. 
We address this problem in Section~\ref{sec sce 2}.

\section{Methodology}

\subsection{A Unified Lyapunov Guidance Framework for Flow Matching}\label{sec method}
To incorporate diverse forms of prior knowledge into the inference stage of flow matching, we propose a unified Lyapunov guidance framework that interprets guidance as a control policy derived from Lyapunov theory. Specifically, we consider the guided vector field as $\vu_t + \vc_t$, where $\vu_t$ is the learned base transport field from noise distribution to data distribution, and $\vc_t$ is an auxiliary control term that steers the dynamics toward a conditional target distribution.
To align with the common Lyapunov condition, we first introduce {\textit{local Lyapunov condition} that formalizes how the controlled dynamics can converge to a desired conditional mode with diverse high-density regions, i.e., attractors. We then propose the first main theorem showing that guided flow matching is equivalent to a Lyapunov control system, with the energy function acting as a Lyapunov function. This equivalence provides a rigorous justification for viewing guidance in generative modeling through the lens of stability theory. 

Inspired by Theorem~\ref{prop ODE}, we extend the traditional Lyapunov condition to the following local Lyapunov condition for multiple equilibria.
\begin{proposition}[\textbf{Lyapunov condition for multi-attractor}]\label{lya cond}
    For the controlled dynamics $\dot{\vx}=\vu_t(\vx)+\vc_t(\vx)$ under controller $\vc_t(\vx_t)$, suppose there exists a continuous differentiable function $V: \mathcal{D} \to R$ that satisfies the following conditions: $\mathrm{(i)}$
	for each local minimum point $\vx^\ast$ of $V$, let $V_{\vx^\ast}(\vx)=V(\vx)-V(\vx^\ast)$, there exists a $\varepsilon$-neighborhood $O(\vx^\ast,\varepsilon)$ such that $V_{\vx^\ast}(\vx)\ge c\Vert\vx-\vx^\ast\Vert^p$,~$\forall\vx\in O(\vx^\ast,\varepsilon)$ for some constants $c,p>0$, $\mathrm{(ii)}$ and $\nabla V(\vx)\cdot\left(\vu_t(\vx)+\vc_t(\vx)\right) < -\delta V$,$\forall\vx\in O(\vx^\ast,\varepsilon)$, for some $\delta>0$. 
	Then, the system is locally exponentially stable at $\vx^\ast$, that is, $\|\vx_t-\vx^\ast\|\le \|\vx_0\|\mathrm{e}^{-\frac{\delta}{p}t}$.
    Here, $V$ is still denoted as a Lyapunov function.
\end{proposition}

The above local condition enables us to view the guidance process in flow matching as the synthesis of a stabilizing controller that enforces local convergence toward high-density regions of a target conditional distribution. With this formulation, the energy function commonly used in conditional generative modeling plays the role of a Lyapunov function that shapes the convergence geometry. Our main theorem is presented as follows.

\begin{theorem}[\textbf{Equivalence between Guided Flow Matching and Lyapunov Control}]\label{main thm}
    For the flow model $\vu_t(\vx_t)$ that generates the probability path $p_t(\vx)$, finding the guidance $\vc_t(\vx_t)$ to the vector field $\vu_t(\vx_t)$ to perform conditional sampling $p_t'(\vx)=\frac{1}{Z_t}p_t(\vx)\mathrm{e}^{-J(\vx)}$ is equivalent to finding the controller that satisfies the local Lyapunov condition, where energy function $J(\vx)$ contributes to Lyapunov function as $V\propto (J+c)$ for some constant $c$, e.g., $V=J$. More specifically, the equivalence means that there exists a control that both generates the guided probability path and satisfies the local Lyapunov stability condition, thereby accelerating the sampling dynamics of guided flows.
\end{theorem}

The detailed proof is provided in Appendix~\ref{proof main}.~The key idea is to rewrite the guided probability path through a weighted continuity equation associated with the target density:
\begin{equation*}\label{eq1 in main}
\nabla\!\cdot\!\big(p'_t\,\vc_t\big)
= p'_t\Big(\vu_t\!\cdot\nabla J + \partial_t\log Z_t\Big).
\end{equation*}
Starting from the reweighted path, one derives a weighted divergence equation for the guidance term, which characterizes how the additional control should modify the base vector field so that the resulting flow is consistent with the desired conditional distribution. Based on this relation, the guidance is further decomposed into a normal component along $\nabla V$ and a tangential component along the level sets of $V$. The normal component is chosen to enforce the Lyapunov decrease condition, while the tangential component is used to match the weighted divergence equation. This yields a control that is simultaneously guidance-compatible and Lyapunov-stable, establishing the equivalence.

Theorem~\ref{main thm} establishes the equivalence between guided flow matching and Lyapunov control, where the guidance term is viewed as a control input and the energy function serves as a Lyapunov function. In particular, it reveals a two-way relation between guidance-compatible controls, characterized by the weighted divergence equation, and Lyapunov-stable controls, characterized by monotone decay of $V$. Although both directions are valid in principle, constructing a guidance-compatible control from a Lyapunov-stable one requires solving the coupled PDE system,
\begin{equation*}
    \begin{aligned}
        \nabla V(\vx)\cdot\vw_t(\vx)&=0,~\forall\vx\in O(\vx^\ast,\varepsilon),~\forall~\text{local minima~}\vx^\ast,\\
        \nabla\cdot (p_t'\vw_t)&=p_t'\phi_t-\nabla\cdot(p_t'\vc_t),
    \end{aligned}
\end{equation*}
which is generally expensive in high dimensions. By contrast, starting from an existing guidance term is much simpler: Theorem~\ref{thm st guarantee} shows that Lyapunov stability can be enforced by a pointwise projection that modifies only the component along $\nabla V$. This observation underlies LyaGuide, which exploits the tractable direction of the equivalence to obtain a guidance term that is both stable and efficient.

To illustrate the wide applicability of the framework, we identify several representative guidance paradigms where the associated energy function and the conditional distribution naturally serve as Lyapunov functions, so that the guidance terms can be interpreted as controllers that minimize $V(\vx)$ during generation.
\begin{proposition}\label{prop1}
The following commonly used guidance strategies in generative modeling can all be interpreted as Lyapunov control within our unified framework:
    \begin{itemize}
    \item \textbf{Classifier Guidance}: Given a trained classifier $p(\vy|\vx)$, the Lyapunov function for guided distribution $p(\vx|\vy)$ specified on conditioner $\vy$ is $V_\vy(\vx)=-\log p(\vy|\vx)$.
    \item \textbf{Reward Guidance}: In reinforcement learning tasks with reward function $R(\vx)$, the Lyapunov function for guided distribution $\tfrac{1}{Z} p_t(\vx) e^{R(\vx)}$ concentrates probability mass in high-reward regions is $V(\vx)=-R(\vx)$.
    \item \textbf{Energy-Based Model (EBM) Guidance}: For a target EBM $p(\vx) \propto e^{-E(\vx)}$, the Lyapunov function is naturally  $V(\vx) = -E(\vx)$.
    \item \textbf{Image inverse problems.} Let the forward operator be $\vy=H(\vx)+\varepsilon$ with $\varepsilon\sim\mathcal{N}(0,\sigma^2 I)$. 
    Then for $p(\vy|\vx)\propto \exp\!\big(-\tfrac{1}{2\sigma^2}\|H(\vx)-\vy\|_2^2\big)$ and a natural Lyapunov function is $
    V_{\vy}(\vx)=\frac{1}{2\sigma^2}\,\|H(\vx)-\vy\|_2^2$, 
    yielding guided sampling $p'_t(\vx)\propto p_t(\vx)\,\exp\!\big(-V_{\vy}(\vx)\big)$ that enforces data consistency~\cite{song2023pseudoinverse}.
\end{itemize}

\end{proposition}

In Appendix~\ref{sec unifying results}, we provide proof of the proposition and further discuss the relationship between existing guidance methods and our framework.

\textbf{Variants of LyaGuide.}~The Proposition~\ref{lya cond} establishes exponential stability of the controlled system, so we denote the corresponding method as \textbf{LyaGuide-ES}. Under a weaker condition that $\delta=0$ in Theorem~\ref{main thm}, the equilibrium remains stable in the sense of asymptotic stability, i.e., $\lim_{t\to\infty}\|\vx_t-\vx^\ast\|=0$. In this case, $V$ is also a valid Lyapunov function. In addition, we can also modify the Lyapunov condition for a stronger exponential convergence. Therefore, two natural variants of LyaGuide arise by either relaxing or strengthening the Lyapunov condition:
\begin{itemize}
% [itemsep=-0.15em,topsep=-0.45em,leftmargin=2em]
    \item \textbf{LyaGuide-AS (Asymptotic Stability).} Setting $\delta=0$ reduces the condition to $\nabla V(\vx)\cdot(\vu_t(\vx)+\vc_t(\vx))<0$, which guarantees asymptotic convergence $\lim_{t\to\infty}\|\vx_t-\vx^\ast\|=0$. This yields a weaker but broadly applicable guidance mechanism.  
    \item \textbf{LyaGuide-CS (Component-wise Stability).} A stronger requirement is that each component marked by subscript $i$ satisfies $\nabla V(\vx)_i (\vu_t(\vx)+\vc_t(\vx))_i \le -\delta V(\vx)$. This enforces descent along all directions and provides stricter stability guarantees for the guided flow.  
\end{itemize}

These two variants offer flexible trade-offs between stability and practical applicability. Empirically, LyaGuide-ES and LyaGuide-CS demonstrate superior performance and greater robustness across tasks, and we therefore recommend them as the default variants in practice (see Appendix~\ref{appen_details}).

\subsection{Pseudo Projection Operation for Lyapunov Guarantee}\label{sec proj}

\begin{figure}
% \vskip -0.17in
	\centering
	\includegraphics[width=0.3\textwidth]{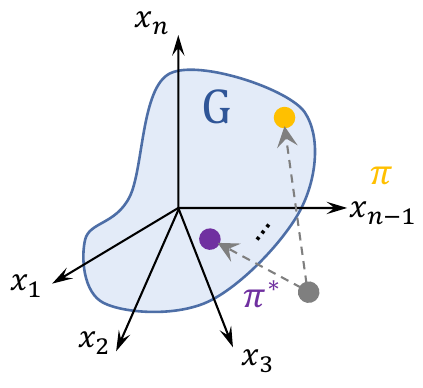}
	\caption{ Illustration of the pseudo projection $\pi$  and exact projection $\pi^\ast$. Here the grey dot is the candidate control $\vc$, the purple dot is the projected element $\pi^\ast(\vc_t)$ of $\vc$ in the target space $\mathcal{U}(V)$, and the yellow dot is the pseudo projected element $\pi(\vc_t)$.}
     
     \label{sketch1}
% \vskip -0.15in
\end{figure}

Given the equivalence between the guided flow matching and Lyapunov control, we show how we can better design the guidance term from the perspective of Lyapunov principle. Before the introduction of Lyapunov guidance, we propose a novel pseudo projection method that efficiently corrects the guidance term learned by neural networks. As in neural control, the candidate guidance function in flow may not rigorously satisfy the Lyapunov condition across the entire state space, since it is often learned from finite data or heuristically constructed. To address this limitation, we propose a projection operator that enforces the Lyapunov condition by projecting any candidate guidance into the admissible set of Lyapunov-stable controls.

\begin{theorem}[Lyapunov Guarantee for Guidance]\label{thm st guarantee} 
 For a candidate controller $\vc$ and the guidance controller space $\mathcal{U}(V)=\{\vu:\nabla V\cdot(\vu_t+\vc_t)+\delta V\le0\}$ that rigorously satisfies the local Lyapunov condition in Proposition~\ref{lya cond}, 
define the projection operator as
\begin{equation*}
	\resizebox{1.0\linewidth}{!}{$
\pi(\vc_t,\mathcal{U}(V)) \triangleq 
\vc_t- \frac{\max\bigl(0,\nabla V(\vx)\cdot(\vu_t(\vx)+\vc_t(\vx))+\delta V(\vx)\bigr)}{\|\nabla V(\vx)\|^2}\,\nabla V(\vx).
$}
\end{equation*}
Then ${\pi}(\vc_t,\mathcal{U}(V))$ is locally Lipschitz continuous and thus the guided flow under ${\pi}(\vc_t,\mathcal{U}(V))$ is well defined, and ${\pi}(\vc_t,\mathcal{U}(V)) \in \mathcal{U}(V)$.
\end{theorem}

The proof is provided in Appendix~\ref{append proof pseudo}. Fig.~\ref{sketch1} shows the idea behind the pseudo projection, that any candidate guidance function can be systematically corrected via the pseudo projection operation, yielding a valid Lyapunov guidance that rigorously satisfies the stability condition across the state space. Therefore, our approach enjoys theoretical guarantees of conditional sampling via Lyapunov stability even when the initial candidate does not. 
Combining Theorem~\ref{main thm} with Theorem~\ref{thm st guarantee}, we obtain a rigorous stability guarantee for flow model with projected guidance.

\textbf{Exact Projection vs. Pseudo Projection.}
We refer to the operator in Theorem~\ref{thm st guarantee} as a \emph{pseudo projection}, because for a feedback controller $\vc$, it is not the exact solution of the classical projection problem~\cite{deenen2021projection}:
\begin{equation*}
  \begin{aligned}
    &\pi^\ast(\vc)\in\arg\min_{\tilde{\vc}}\ \Vert\vc-\tilde{\vc}\Vert_{C(\mathbb{R}^d)},\\
    &\text{s.t.}\quad \nabla V(\vx)\cdot\left(\vu_t(\vx)+\vc(\vx)\right)\le -\delta V,\quad \forall \vx\in\mathcal{D}.
  \end{aligned}
\end{equation*}
In general, obtaining a closed-form expression of the exact projection operator $\pi$ is intractable. A common alternative is to solve a simplified quadratic program (QP) at each state $\vx$~\cite{chow2019lyapunov}:
\begin{equation*}
  \begin{aligned}
    &\tilde{\pi}(\vc)(\vx)\in\arg\min_{\tilde{\vc}(\vx)} \ \Vert\vc(\vx)-\tilde{\vc}(\vx)\Vert,\\
   &\text{s.t.}\quad \nabla V(\vx)\cdot\left(\vu_t(\vx)+\vc(\vx)\right)\le -\delta V(\vx).
  \end{aligned}
\end{equation*}
Although this guarantees the Lyapunov condition for each state $\vx$, it requires solving an online optimization problem. In the context of flow matching, the dimensionality of the state space can be very high (e.g., images, molecular structures, or protein conformations). In such cases, solving a QP in every integration step would be prohibitively expensive.  In contrast, the pseudo projection in Theorem~\ref{thm st guarantee} has an explicit analytical form, making it more efficient in practice while still ensuring that the guidance satisfies the Lyapunov condition.

% In the context of flow matching, the dimensionality of the state space can be very high (e.g., images, molecular structures, or protein conformations). In such cases, solving a QP in each integration step would be prohibitively expensive. The pseudo projection offers a tractable alternative with closed-form updates, enabling stability enforcement during sampling without introducing substantial computational overhead. This makes the proposed approach particularly well suited for high-dimensional generative modeling tasks.

% {\color{red}
% \textbf{Intuitive explanation of Theorems}.
% Our Theorem~\ref{main thm} shows that Lyapunov-stable controls and guidance-compatible controls are theoretically interchangeable, but the two directions differ greatly in difficulty. Converting a Lyapunov control into a guidance term requires solving auxiliary PDEs (see Appendix~\ref{proof main}), whereas converting a guidance term into a Lyapunov-stable one is much easier: a simple pointwise projection onto the Lyapunov-stable cone guarantees monotone decrease of 
% $V$ while preserving the original guidance structure, as shown in Theorem~\ref{thm st guarantee} later.
% This is exactly the advantage exploited by LyaGuide—the projection step provides an efficient, PDE-free way to obtain a control that satisfies both stability and guidance properties.
% }

\subsection{Lyapunov Guidance for Different Scenarios}\label{sec sce}

Building on the equivalence between guided flow matching and Lyapunov control established in Section~\ref{sec method}, we now specify the guidance policies for two different practical scenarios. 
% The distinction lies in whether prior knowledge can be explicitly formulated: in Scenario~1, the Lyapunov function $V$ admits an analytical expression, and guidance can be synthesized directly (or learned from $V$); but in Scenario~2, the Lyapunov function is unknown and must be learned from data, which requires us to learn both $V$ and the associated control.

\subsubsection{Scenario 1: Explicit Prior Knowledge}\label{sec sce 1}
When domain knowledge can be analytically represented, the Lyapunov function $V(\vx)$ is directly available.  For such tasks, $V(\vx)$ acts as a certificate function that encodes domain-specific constraints, while the control term $\vc(\vx,t)$ is designed according to the Lyapunov condition in Proposition~\ref{lya cond}, ensuring that the guided flow is exponentially stable toward the high-density regions of the conditional distribution. 

% \begin{equation}\label{eq lya cond}
% \nabla V(\vx)\cdot \left(u_t(\vx)+c_t(\vx)\right) \leq -\delta V(\vx),
% \end{equation}

According to the projection operation described in Section~\ref{sec proj}, we begin with a candidate guidance function and then enforce the Lyapunov condition by projecting it onto the admissible function space. Within the flow matching literature, a common choice of candidate is based on the gradient of the Lyapunov function, 
$\vc_t(\vx) \propto -\nabla V(x)$, 
which is closely related to the score function of the prior energy distribution~\cite{song2019generative}. The complete procedure is summarised in Algorithm~\ref{algo1}.

\subsubsection{Scenario 2: Implicit Prior Knowledge via Few-Shot Learning}\label{sec sce 2}

When explicit formulations of prior knowledge are unavailable, we assume access to a small set of preference data–score pairs $\{(\vx_i,V_i)\}_{i=1}^n$, where $V_i$ represents a task-specific score or energy evaluated at $\vx_i$. Our first goal is to recover a Lyapunov function $V_{\vtheta_V}(\vx)$ from these pairs by minimizing a supervised regression loss: $
\mathcal{L}_V(\vtheta_V) = \frac{1}{n}\sum_{i=1}^n \big(V_{\vtheta_V}(\vx_i)- V_i\big)^2.$

\paragraph{Local-minima-aware training of $V$}
According to Proposition~\ref{lya cond}, the Lyapunov function $V$ should correctly identify task-relevant local minima, i.e., low-energy and high-density regions where the Lyapunov guidance becomes active.
To bias the learner toward these regions when only few-shot supervision $(\vx_i, V_i)$ is available, 
we adopt a soft importance weighting on the targets,
\[
w_i \;=\; \frac{\exp(-\alpha\,V_i)}{\sum_{j=1}^n \exp(-\alpha\,V_j)},
\]
which places greater emphasis on low-$V$ samples. 
We then estimate $V_{\theta_V}$ via a weighted regression, optionally augmented with a smoothness regularizer,
so that the learned potential captures the correct local minima structure needed for Lyapunov guidance.

\begin{equation}\label{eq sce 2 loss} 
    \mathcal{L}_V(\vtheta_V)=
\frac{1}{n}\sum_{i=1}^n w_i\big(V_{\theta_V}(\vx_i)-V_i\big)^2.
\end{equation}

\paragraph{Control design.}
Once a valid Lyapunov function $V_{\vtheta_V}$ is obtained, the corresponding guidance control can then be derived in two ways: either by explicit synthesis according to the Lyapunov framework,
\begin{equation*}
\vc(\vx,t) = \arg\min_{\vc} \;\;\nabla V_{\vtheta_V}(\vx)\cdot(\vu_t(\vx)+\vc(\vx,t)) + \delta V_{\vtheta_V}(\vx),
\end{equation*}
or by integrating $V_{\vtheta_V}$ into existing guidance methods (e.g. classifier or reward guidance) to regularize their dynamics via a Lyapunov-inspired penalty. In this work we employ the training algorithm $g_\phi$ posed in~\cite{fengguidance} for learning the guidance.

This two-stage design separates the estimation of the implicit energy landscape from the design of the guidance control, enabling user-provided supervision to be incorporated flexibly into flow matching. 
Although we present a two-stage procedure here (first learning $V$, then designing $\vc$), in practice $V$ and $\vc$ can also be optimized jointly under a Lyapunov-inspired loss, as applied in~\cite{zhang2022neural,zhang2024fessnc,yang2025neural}

\section{Experiments}\label{sec experiments}

We evaluate LyaGuide on a diverse set of benchmarks covering both low-dimensional and high-dimensional generative tasks, including synthetic conditional generation, image inverse problems, offline reinforcement learning planning, and energy-based modeling. Across all experiments, our goal is to assess whether the proposed Lyapunov-guided correction can consistently improve existing guidance methods without retraining the underlying generative model. Unless otherwise specified, we compare each original guidance baseline with its LyaGuide-enhanced counterpart under the same base flow model and task setting, so that the effect of the proposed projection mechanism can be isolated.

\subsection{Experimental Setup}
For synthetic experiments, we consider the standard 2D benchmarks used in prior work, including uniform-to-8Gaussians and circle-to-S-curve. Following the setup in~\cite{fengguidance}, the base flow model is trained with displacement interpolation using a 4-layer MLP with hidden dimension 128, optimized by Adam with learning rate $10^{-4}$ for 20k iterations and batch size 512. In Scenario 1, we compare against a wide range of existing guidance methods reported in~\cite{fengguidance} as our baseline: contrastive energy guidance (CEG)~\cite{lu2023contrastive}, Monte Carlo guidance ($g^{\mathrm{MC}}$), learned guidance ($g_\phi$), covariance-based approximations ($g^{\mathrm{cov\text{-}A}}$ and $g^{\mathrm{cov\text{-}G}}$), and simplified Monte Carlo ($g^{\mathrm{sim\text{-}MC}}$). We also include contrastive guidance, which drives the flow toward unexpected regions, and gradient-based guidance commonly used in generative modeling~\cite{zhang2024generalized} for comparison. Each of these methods is first applied in its original form, and then we treat their guidance terms as candidates in Algorithm~\ref{algo1} and refine them with our proposed LyaGuide-ES. We also provide some detailed analysis results for LyaGuide-AS and LyaGuide-CS in Appendix~\ref{appen_details}.  In the few-shot setting (Scenario 2), the supervision size varies from 128 to 1024 preference data--score pairs, and the Lyapunov candidate $V_\theta$ is parameterized by a 3-layer MLP with hidden width 64 and trained by weighted regression.

For image inverse problems, we follow the standard protocol on CelebA-HQ for box inpainting, Gaussian deblurring, and super-resolution. We use both conditional flow matching (CFM) and optimal-transport conditional flow matching (OT-CFM) as the base models, and evaluate reconstruction quality using FID, LPIPS, PSNR, and SSIM on 3000 test samples. Baseline guidance methods are compared with and without LyaGuide under identical task settings, and the hyperparameters of the pseudo-projection operator are fixed across all inverse-problem tasks to demonstrate robustness without additional tuning.

For RL planning, we follow the standard D4RL Locomotion protocol and evaluate guided planning performance under both CFM and OT-CFM backbones. Following~\cite{fengguidance}, performance is measured by the average normalized return over five runs. In all experiments, we retain the original implementation of each baseline guidance method and introduce LyaGuide only as a post-training correction, thereby isolating its effect on stability, convergence, and sample quality.

\begin{figure*}
% \vskip -0.17in
	\centering
	\includegraphics[width=1.0\textwidth]{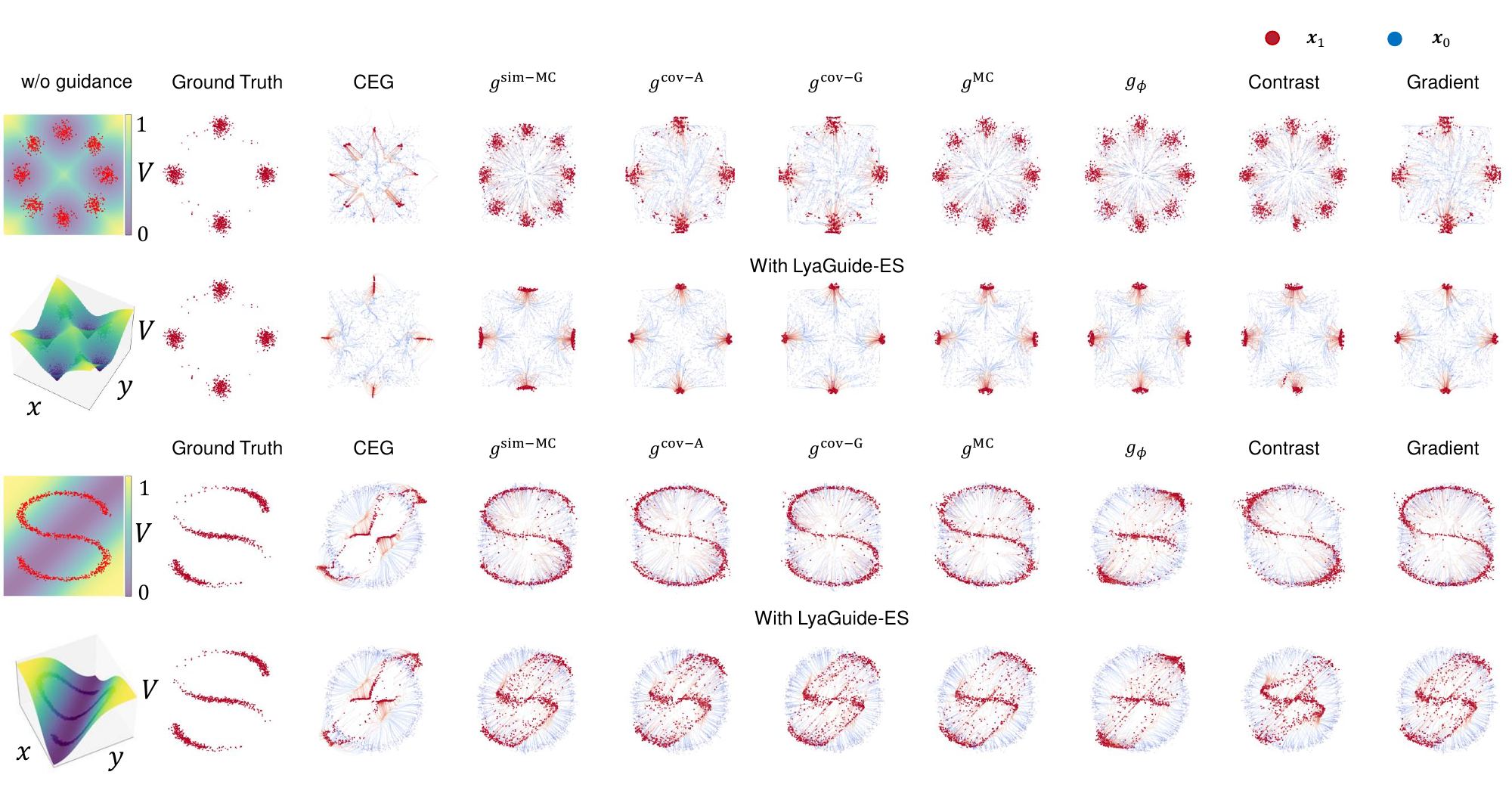}
	\caption{Scenario 1 results on synthetic dataset. For each target distribution, the top (resp. bottom) row correspond to the methods without (resp. with) LyaGuide. We visualize the start/end
 points and the flow trajectories. LyaGuide siginificantly improves the performance across different methods. 
     }
     \label{fig_lya_s1}
% \vskip -0.25in
\end{figure*}

\subsection{Synthetic Benchmarks}\label{sec:synthetic}

\paragraph{Scenario 1}~We first evaluate our approach on the synthetic datasets introduced in~\cite{fengguidance}, where the source distribution $p_0$ is chosen as uniform (resp. circle) distribution and the target distribution $p_1$ is an 8-Gaussian mixture (resp. S-curve). For each dataset, we design a task-specific energy function $V$ to encode prior knowledge, as shown in Fig.~\ref{fig_lya_s1} (first column on the left).

\begin{figure}
% \vskip -0.17in
	\centering
	\includegraphics[width=0.5\textwidth]{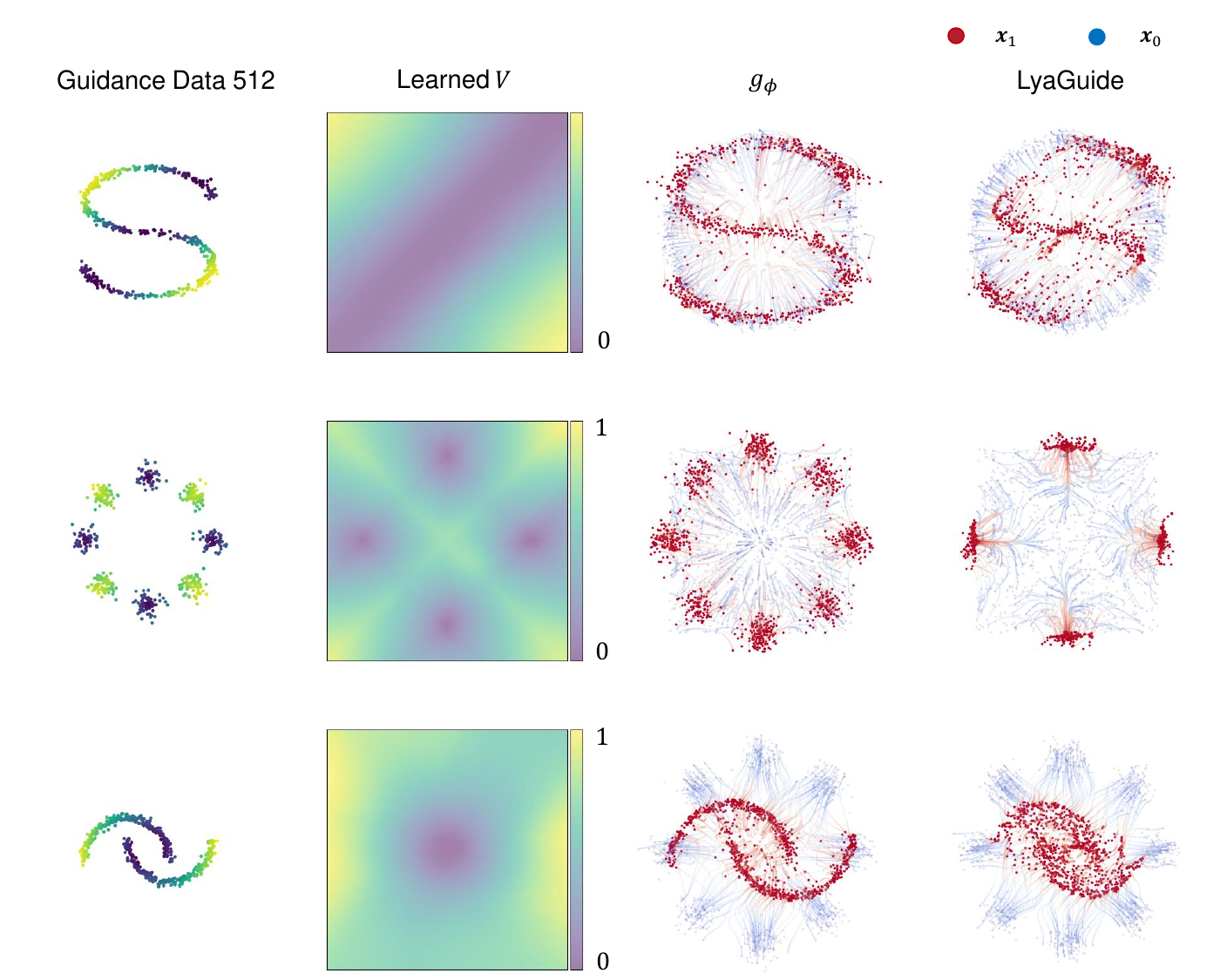}
	\caption{Scenario 2 results with dataset size $=512$.}
     \label{fig_lya_s2}
% \vskip -0.25in
\end{figure}

% \begin{figure}
% \vskip -0.17in
% 	\centering
% 	\includegraphics[width=0.6\textwidth]{figs/lya_s2_1024.pdf}
% 	\caption{Results trained on dataset with size $512$. 
%      }
%      \label{fig_lya_s2}
% % \vskip -0.25in
% \end{figure}

As shown in Fig.~\ref{fig_lya_s1}, all methods benefit from our Lyapunov guidance, where the guided trajectories become more stable and align more closely with the ground truth distribution. Specifically, for the first task (8-Gaussian mixture), we can see that most of the methods fail to sample correct trajectories without the guidance of LyaGuide, of which the best performance achieved by $g^{\mathrm{cov\text{-}A}}$, $g^{\mathrm{cov\text{-}G}}$ and $Gradient$ is still defective, while noise at the four corners of the diagonals drops off drastically or even disappears after application of LyaGuide. Similar effect can be observed for the second task (S-curve distribution), where the LyaGuide optimized methods exhibit sharper mode coverage and fewer spurious samples compared to the original versions. These results demonstrate that LyaGuide provides an effective and practical way to improve guidance quality, consistently achieving gains across all baselines and energy functions.

\paragraph{Boundary effects of Pseudo Projection}~We also note that there are two subtleties in the second task. First, CEG+Lya shows little difference with CEG, as the original samples already lie in the trough of $V$; since our projection encourages trajectories toward the trough, no further adjustment occurs. Second, Contrast+Lya does not preserve the exact masked S-curve shape. Instead, samples concentrate along the boundary of the trough of $V$, consistent with our projection principle: enforcing the Lyapunov inequality $\{\vx:\nabla V(\vx)\cdot (\vu_t(\vx)+\vc_t(\vx)) \leq -\delta V(\vx)\}$. Thus, data near the peaks of $V$ are guided toward the nearest trough boundary, rather than the lower-left or upper-right arms of the S-curve.

\paragraph{Scenario 2}
We next evaluate LyaGuide in Scenario~2, where explicit prior knowledge is not available and only a small set of preference data--score pairs is provided. 
Following the setup in Section~\ref{sec sce 2}, we learn a Lyapunov function $V_{\vtheta_V}$ first and then derive the guidance control either through explicit synthesis or by integrating into existing methods. 
We consider synthetic benchmarks with three tasks, and vary the number of supervision pairs $\{(\vx_i,V_i)\}$ from 128 to 1024. 

As shown in Fig.~\ref{fig_lya_s2}, even with very sparse supervision, LyaGuide successfully recovers meaningful Lyapunov landscapes whose minima align with the high-density regions of the target distribution. 
Compared to the baseline $g_\phi$, the LyaGuide method enforces stability throughout the trajectory, yielding better guided flows that concentrate around the correct modes while suppressing spurious samples. We further perform the ablation study on data size in Appendix~\ref{appen_details}.

\paragraph{Acceleration of Inference Speed}
To evaluate how the sampling horizon of flow matching influences the behaviour of LyaGuide, we perform an early-inference-termination study on the 8-Gaussian mixture task. Instead of integrating the flow dynamics up to 
$t=1$, we interrupt the evolution at intermediate times and compute the Wasserstein-2 distance between the current particle distribution and the target mixture. Figure~\ref{fig_lya_sample steps} reports the mean and standard deviation over five independent trials. Across all eight guidance mechanisms, LyaGuide exhibits substantially faster convergence than their unguided counterparts. Even at early inference times (e.g., $t\approx 0.4$), the guided trajectories have already contracted much closer to the target distribution, while the baseline flows remain noticeably farther away. {\color{black}This indicates that the Lyapunov-structured correction not only improves final sample quality but also accelerates the transient approach to the target distribution}. In addition, the performance gap remains robust across guidance types, demonstrating that LyaGuide consistently enhances the efficiency of flow-matching inference and reduces the reliance on long integration horizons. {\color{black}The acceleration effect comes from projection-induced contraction along the normal direction of the Lyapunov level sets, which suppresses oscillatory motion along the tangent direction.}

{\paragraph{Ablation study}
We further examine the influence of the scaling parameters $\delta$ and $
k$ on the behaviour of LyaGuide. The results are summarised in Fig.~\ref{fig ablation} and Table~\ref{tab ablation}. The parameter $\delta$ affects both the Lyapunov convergence rate and the projection strength, making it a global hyperparameter. Empirically, LyaGuide is robust across a wide range of $\delta$, where larger values lead to stronger contraction toward low-energy regions but may reduce exploration. A moderate choice of $\delta\in(0,1]$ consistently provides a good balance between stability and diversity. The coefficient $k$ influences only gradient-based guidance methods by adjusting the strength of the initial guidance term $-k\nabla V$. Increasing $k$ generally improves the stability and quality of gradient-based guidance without affecting other methods. Overall, the method demonstrates low sensitivity to both parameters within reasonable ranges.

\begin{figure*}[htb]
% \vskip -0.17in
	\centering
	\includegraphics[width=1.0\textwidth]{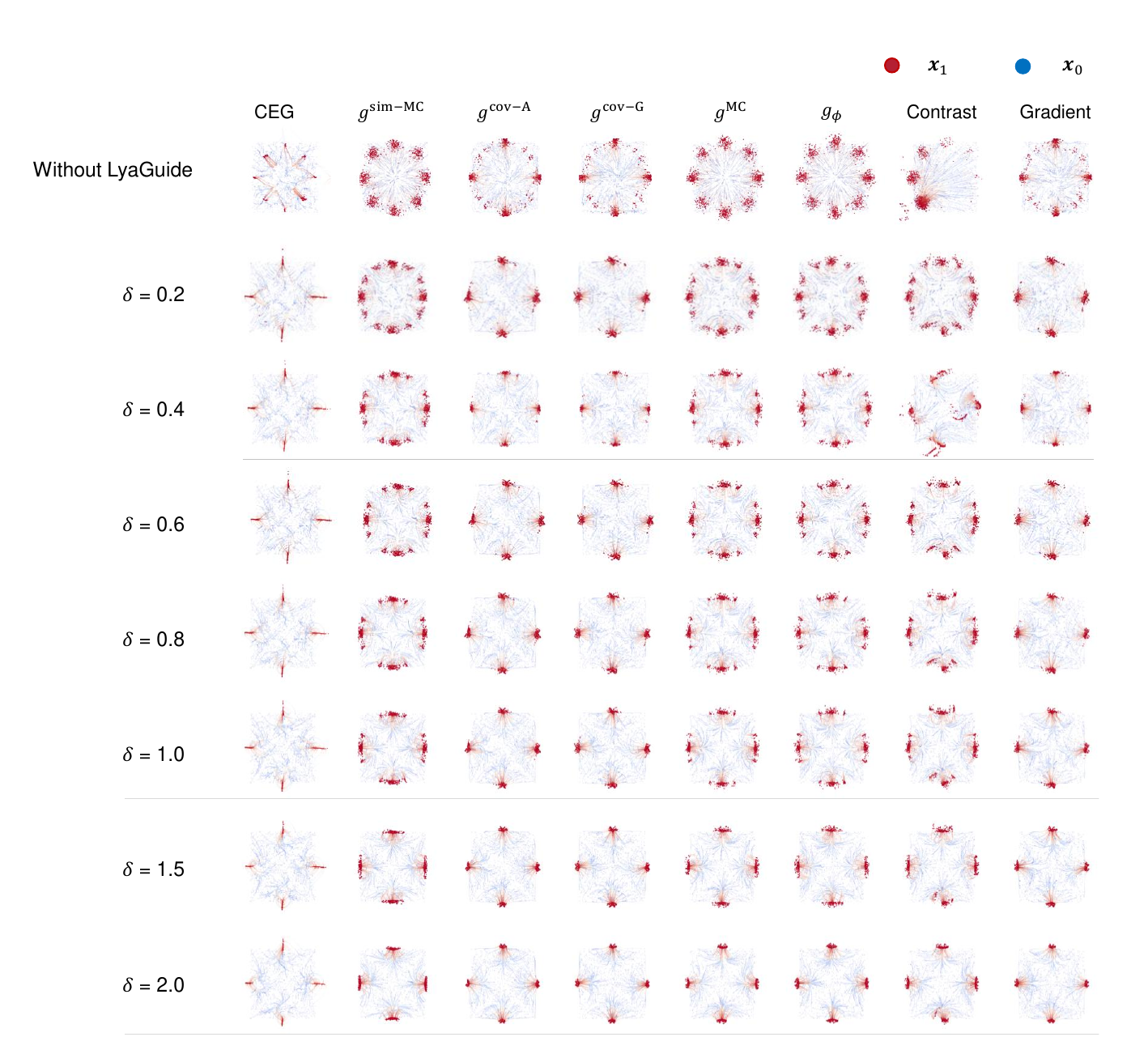}
	\caption{Ablation study in 8-Gaussian task. We investigate the effect of the Lyapunov convergence rate $\delta$ to the performance of the LyaGuide sampling. The first row corresponds to the results of the orginal methods without LyaGuide, the other rows correspond to the results of LyaGuide with different choices of $\delta$.
     }% \vskip -0.25in 
     \label{fig ablation}
\end{figure*}

\begin{table*}[htb]
% \vskip -0.15in
	\centering
	\caption{Ablation study on gradient-based method in 8-Gaussian generation. Wasserstein-2 distances between the target data and the sampled data under different combinations of the hyperparameter $\delta$ and $k$ are shown.}
	% \resizebox{\linewidth}{!}{
	\setlength{\tabcolsep}{0.8mm}{
		\begin{tabular}{ccccccccccc}
			\toprule
			\toprule
			\multicolumn{2}{c}{\multirow{2}{*}{}}   & \multicolumn{3}{c}{Original Gradient Method} &
			\multicolumn{3}{c}{LyaGuide-ES} & \multicolumn{3}{c}{LyaGuide-CS}   \\
			\cmidrule(lr){3-5}  \cmidrule(lr){6-8}
			\cmidrule(lr){9-11} 
		&	& $k=0.5$  & $k=1.0$ &  $k=1.5$    &$k=0.5$  & $k=1.0$ &  $k=1.5$   &$k=0.5$  & $k=1.0$ &  $k=1.5$     \\
			\midrule 
			&{$\delta=0.2$}   & 0.29 & 0.24 & 0.18 & 0.28 & 0.25 & 0.16 & 0.33 & 0.2 & 0.34	 \\
        	&{$\delta=0.4$}   & 0.34 & 0.19 & 0.25 & 0.28 & 0.28 & 0.3 & 0.33 & 0.24 & \textbf{0.09}	 \\
            	&{$\delta=0.6$}   & 0.42 & 0.19 & 0.28 & 0.39 & 0.27 & 0.31 & 0.41 & 0.17 & 0.18	\\
                &{$\delta=0.8$}  & 0.3 & 0.33 & 0.25 & 0.3 & 0.33 & 0.25 & 0.4 & 0.24 & 0.21 \\
           &{$\delta=1.0$} & 0.32 & 0.26 & 0.26 & 0.35 & 0.22 & 0.24 & 0.32 & 0.2 & 0.17 \\
           &{$\delta=1.5$}  & 0.29 & 0.17 & 0.16 & 0.33 & 0.16 & 0.2 & 0.42 & 0.27 & 0.25 	 \\
           &{$\delta=2.0$}  & 0.45 & 0.19 & 0.29 & 0.37 & 0.23 & 0.31 & 0.41 & 0.2 & 0.34\\
			\midrule
			\bottomrule
		\end{tabular}
	}
	\label{tab ablation}
    % \vskip -0.2 in
\end{table*}

\begin{figure*}
% \vskip -0.17in
	\centering
	\includegraphics[width=1.0\textwidth]{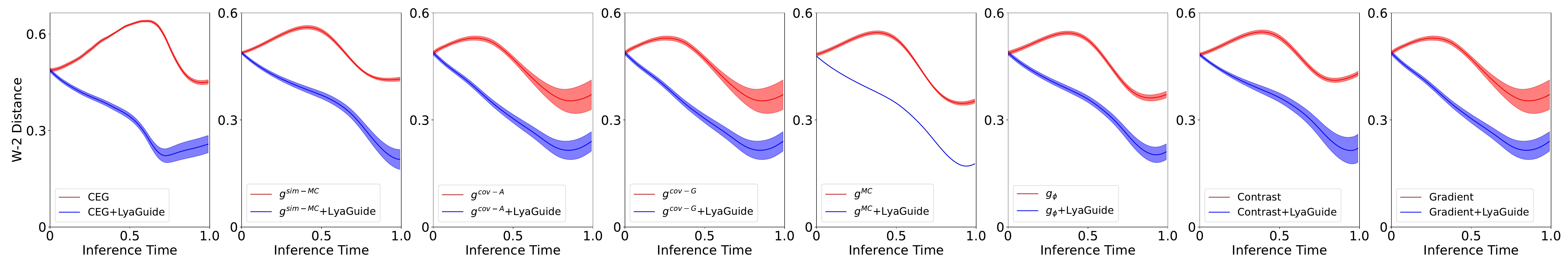}
	\caption{Inference-time vs. Wasserstein-2 distance on the 8-Gaussian mixture task (1024 test samples).
Curves show the mean over five runs, and shaded regions denote standard deviations.
     }
     \label{fig_lya_sample steps}
% \vskip -0.25in
\end{figure*}

}

\subsection{Image Inverse Problems}
We further validate LyaGuide on high-dimensional image inverse problems, which have become a standard benchmark for assessing guidance quality in flow matching~\cite{song2023pseudoinverse, fengguidance}. 
The objective is to reconstruct a clean image $\vx$ from a corrupted observation $\vy=H(\vx)+\varepsilon$, where $H$ is a known degradation operator and $\varepsilon$ denotes Gaussian noise. 
In contrast to pseudoinverse-guided diffusion~\cite{song2023pseudoinverse}, our projection-based scheme offers a lightweight alternative that strictly enforces stability during sampling.

We conduct experiments on the CelebA-HQ dataset under the box-inpainting setting, where a central square region of each image is masked and the generative model is required to reconstruct the missing content from the visible context. We adopt the standard evaluation protocol used in~\cite{fengguidance} and report FID, LPIPS, PSNR, and SSIM over 3000 test samples. As shown in Table~\ref{tab1}, LyaGuide consistently improves all baseline guidance methods across these metrics, indicating better reconstruction fidelity, perceptual quality, and sample realism. Beyond box inpainting, we further evaluate the proposed framework on two additional inverse problems, namely super-resolution and Gaussian deblurring. Similar performance gains are observed in both cases, demonstrating that the benefit of Lyapunov-guided correction is not restricted to a single degradation model. Detailed quantitative results for these additional tasks are provided in Appendix~\ref{appen_details}.

\begin{table*}[htb]
% \vskip -0.15in
	\centering
	\caption{Image inverse problem results on CelebA-HQ (Box Inpainting task). The best and runner-up results are highlighted with \textbf{bold} and \underline{underline}, respectively.}
	\resizebox{\linewidth}{!}{
		\begin{tabular}{cccccccccccccccc}
			\toprule
			\toprule
			\multicolumn{2}{c}{\multirow{2}{*}{}}   & \multicolumn{4}{c}{Original Methods} &
			\multicolumn{4}{c}{LyaGuide-ES} & \multicolumn{4}{c}{LyaGuide-CS}   \\
			\cmidrule(lr){3-6}  \cmidrule(lr){7-10}
			\cmidrule(lr){11-14} 
		&	&FID $\downarrow$  &LPIPS $\downarrow$ &PSNR  $\uparrow$ &SSIM $\uparrow$  &FID $\downarrow$  &LPIPS $\downarrow$ &PSNR  $\uparrow$ &SSIM $\uparrow$ &FID $\downarrow$  &LPIPS $\downarrow$ &PSNR  $\uparrow$ &SSIM $\uparrow$   \\
			\midrule 
		\multicolumn{1}{c}{\multirow{3}{*}{OT-CFM}}	&{$g^{\text{cov-A}}$}   &\underline{7.3387}  &0.1907  &25.5984  &0.8431   &7.4039 	&\underline{0.1898} 	&{25.6935} 	&\underline{0.8442}  &\textbf{6.4722}	&\textbf{0.1739}	&25.6217	&\textbf{0.8562} \\
        	&{$g^{\text{sim-A}}$}   &19.8569 	&0.2309 	&\underline{26.4127} 	&0.7950  &19.8852 	&0.2309 	&\textbf{26.4741} 	&0.7950  &12.7344	&0.1921	&26.3142	&0.842 \\
            	&{$\Pi$GDM}   &30.3839 	&0.3200 	&20.8253 	&0.7193  &30.3839 	&0.3200 	&20.8253 	&0.7193  &26.785	&0.291	&21.3654	&0.7473 \\
                &{$g^{\text{MC}}$}  &18.6635 	&0.2391 	&26.9492 	&0.8124   &24.1950 &0.2396 &26.9624 	&0.8129 & 16.0392 	&0.2283 &	27.0681 	&0.8200 \\
                \midrule

		\multicolumn{1}{c}{\multirow{5}{*}{CFM}}	&{$g^{\text{cov-A}}$}   &7.6629 	&0.1922 	&25.8612 	&0.8414  &7.6819 	&0.1918 	&25.9367 	&0.8419  &\underline{6.9783}	&\underline{0.1764}	&25.8155	&\textbf{0.854} \\
 
  &{$g^{\text{sim-A}}$}   &9.7060 	&0.1935 	&26.0867 	&0.8263  &9.6002 	&0.1935 	&26.1094 	&0.8263  &\textbf{6.8763}	&\textbf{0.1629}	&26.0592	&\underline{0.8508} \\
   &{$g^{\text{cov-G}}$}  &19.8022 	&0.2379 	&\underline{27.0087} 	&0.8138  &24.3271 	&0.2379 	&27.0017 	&0.8145   &22.0999 	&0.2293 	&\textbf{27.0301} 	&0.8206  \\
			&{$\Pi$GDM} &19.0847 	&0.2323 	&25.6418 	&0.8093  &19.0619 	&0.2336 	&25.9002 	&0.8074  &14.1461	&0.2003	&26.112	&0.8439 \\
			% \cdashline{2-11}[3pt/3pt]
			\midrule
			\bottomrule
		\end{tabular}
	}
	\label{tab1}
    % \vskip -0.2 in
\end{table*}

\subsection{LyaGuide on RL Planning}
To further strengthen the empirical validity of our framework and broaden the range of benchmarks, we additionally evaluate LyaGuide on standard offline RL planning tasks in the D4RL Locomotion suite~\cite{fu2020d4rl}, following the experimental protocol used in prior work~\cite{janner2022planning,fengguidance,chen2024flow}. In this setting, generative models are used as planners by sampling from reward-reweighted distributions proportional to $\exp(R(x))$~\cite{levine2018reinforcement,janner2022planning}, where $R(x)$ denotes the return of a trajectory or action sequence. This formulation naturally aligns with our guidance framework by setting the energy function as $J(x)=-R(x)$, so that reward-guided planning can be interpreted as a special case of Lyapunov-guided conditional generation. We directly compare each baseline guidance method and its LyaGuide-enhanced counterpart under both CFM and OT-CFM generative models, and report the average normalized score over five runs as the evaluation metric.

The results are summarized in Table~\ref{tabRL}, and show consistent performance improvements across nearly all tasks and guidance types. In particular, LyaGuide yields higher normalized returns for HalfCheetah, Hopper, and Walker2d under the Medium-Expert, Medium, and Medium-Replay settings. Averaged over all tasks, LyaGuide improves the strongest baseline by a clear margin under both CFM and OT-CFM. These results indicate that the proposed Lyapunov-guided correction improves the contractive behavior of guidance dynamics in planner-based RL, leading to more reliable trajectory optimization and higher-quality action-sequence generation.

% To further strengthen the empirical validity of our framework and enhance the sufficiency of benchmarks, we additionally evaluate LyaGuide on standard offline RL planning tasks in the D4RL Locomotion suite, following the experimental protocol used in prior work~\cite{fengguidance}. In this setting, generative models are employed as planners by sampling from distributions proportional to $\exp(R(x))$, where $R$ denotes the environment return. We directly compare each baseline guidance method and its LyaGuide-enhanced counterpart under both CFM and OT-CFM generative models, using the average normalized score over five runs as the evaluation metric.

% The results are summarized in Table~\ref{tabRL}, and demonstrate a consistent performance improvement across nearly all tasks and guidance types. In particular, LyaGuide yields higher normalized returns for HalfCheetah, Hopper, and Walker2d across Medium-Expert, Medium, and Medium-Replay settings. Averaged over all tasks, LyaGuide improves the best-performing baseline by a clear margin under both CFM and OT-CFM. These findings confirm that LyaGuide enhances the contractive behaviour of guidance dynamics in planner-based RL, leading to more reliable and higher-quality action-sequence generation.

\begin{table*}[htb]
% \vskip -0.15in
	\centering
	\caption{Results of the D4RL Locomotion experiments. The best results are highlighted with \textbf{bold}.}
	\resizebox{\linewidth}{!}{
		\begin{tabular}{ccccccccccccccc}
			\toprule
			\toprule
			\multicolumn{3}{c}{\multirow{2}{*}{}}   & \multicolumn{5}{c}{CFM} &
			\multicolumn{5}{c}{OT-CFM} &    \\
			\cmidrule(lr){4-8}  
			\cmidrule(lr){9-13} 
		& &	&$g^{\text{cov-A}}$  &$g^{\text{cov-G}}$ &$g^{\text{sim-MC}}$   &$g^{\text{MC}}$   &$g_{\phi}$   &$g^{\text{cov-A}}$  &$g^{\text{cov-G}}$ &$g^{\text{sim-MC}}$   &$g^{\text{MC}}$   &$g_{\phi}$   \\
			\midrule 
		\multicolumn{1}{c}{\multirow{10}{*}{Without LyaGuide}}& \multicolumn{1}{c}{\multirow{3}{*}{Medium-Expert}}	&HalfCheetah  & 44.9 &50.3 	&62.7 	&58.9 	&43.0 	&50.4 &	49.3 &	47.7 	&80.2 	&60.5  \\
        	& &Hopper   &88.9 	&84.3 	&96.9 	&103.5 	&84.7 	&92.6 	&111.1 	&74.9 	&108.4 	&75.5  \\
            &	&Walker2d   &64.5 	&95.7 	&83.7 	&86.2 	&78.0 &	76.4 &	75.8 	&94.4 	&102.5 	&56.5  \\
                  \midrule
     		& \multicolumn{1}{c}{\multirow{3}{*}{Medium}}	&HalfCheetah  &41.1 	&43.1 	&42.5 	&40.6 	&43.1 	&43.4 	&42.3 	&34.4 	&41.6 	&42.4  \\
        	& &Hopper   &67.5 	&83.6 	&73.6 	&71.8 &	69.2 &	70.6 &	67.3 &	63.7 &	70.9 	&66.1  \\
            &	&Walker2d   &75.5 	&68.5 	&74.7 	&78.2 	&50.8 	&75.4 &	77.4 	&77.3 	&78.7 	&67.6  \\
                  
                \midrule
	& \multicolumn{1}{c}{\multirow{3}{*}{Medium-Replay}}	&HalfCheetah  & 35.4 &	33.8 &	26.1 	&34.3 	&29.7 	&30.1 	&31.2 &	19.5 	&34.3 	&29.2   \\
       & 	&Hopper   &44.7 	&50.3 	&46.4 	&56.2 	&44.9 &	48.4 	&57.5 &	49.9 &	63.4 	&50.3   \\
            &	&Walker2d   &47.5 	&48.1 	&47.4 &	53.5 	&38.6 &	53.6 &	45.7 &	30.8 	&59.0 	&43.3   \\
                  
                \midrule
              &  \multicolumn{2}{c}{\multirow{1}{*}{Average of Baselines}}	  &56.7 	&61.9 	&61.5 	&64.8 	&53.5 &	60.1 	&61.9 	&54.7 &	71.0 	&54.6   \\
                  \midrule
		\multicolumn{1}{c}{\multirow{10}{*}{With LyaGuide}}&	\multicolumn{1}{c}{\multirow{3}{*}{Medium-Expert}}	&HalfCheetah  & 56.7 	&65.6 	&58.8 	&88.0 	&58.5 	&50.6 	&61.6 	&57.9 	&88.6	&64.2   \\
        	& &Hopper   &102.6 	&104.2 	&84.2 	&112.2 	&85.8 	&103.5 	&102.8 	&75.7  &111.3		&99.4  \\
            &	&Walker2d   &78.4 	&84.4 	&80.4 	&98.4 	&92.5 	&79.7 	&88.0 	&87.3 		&105.5 &76.1   \\
                  \midrule
     		& \multicolumn{1}{c}{\multirow{3}{*}{Medium}}	&HalfCheetah  &42.3 	&42.32 	&43.7 	&43.0 	&42.6 	&40.6 	&41.8 	&42.0 	&44.3	&43.1  \\
        	& &Hopper   &81.8 	&75.8 	&75.8 	&75.9 	&71.5 	&74.8 	&74.7 	&61.8 	&73.0	&72.0   \\
            &	&Walker2d   &72.9 	&76.6 	&75.8 &79.9 	&79.1 	&70.8 	&76.8 	&78.5 	&	79.6 &73.8   \\
                  
                \midrule
	& \multicolumn{1}{c}{\multirow{3}{*}{Medium-Replay}}	&HalfCheetah  & 32.6 	&36.8 	&35.7 	&38.5 	&30.4 	&28.9 	&35.5 	&26.5 	&41.3	&25.3    \\
      &  	&Hopper   &58.2 	&64.1 	&54.4 	&59.5 	&53.5 	&60.8 	&59.6 	&59.5 	&73.4	&47.6    \\
        &    	&Walker2d   &51.3 	&64.9 	&45.0 	&65.6 	&54.7 	&62.1 	&53.4 	&52.2 	&	69.0&43.0    \\
                  
                \midrule
            &    \multicolumn{2}{c}{\multirow{1}{*}{Average of LyaGuide}}	  &\textbf{64.1} 	&\textbf{68.3} 	&\textbf{61.5} 	&\textbf{73.5} 	&\textbf{63.2} 	&\textbf{63.5} 	&\textbf{66.0} 	&\textbf{60.5} 	&\textbf{76.2}	&\textbf{60.5}   \\
			% \cdashline{2-11}[3pt/3pt]
			\midrule
			\bottomrule
		\end{tabular}
	}
	\label{tabRL}
    % \vskip -0.2 in
\end{table*}

\subsection{Performing LyaGuide to EBM tasks}
In Prop.~\ref{prop1} we show that our LyaGuide framework works to the EBM tasks. To validate the efficacy, we evaluate the effect of applying LyaGuide  on the vector field used in Energy Matching (EM)~\cite{balcerak2025energy}, while keeping all experimental configurations identical to the original EM setup. Recall that EM sampling consists of two stages: a deterministic gradient-driven ODE phase followed by a stochastic Langevin refinement phase. In our method, we apply LyaGuide directly to the underlying vector field $-\nabla V$ used in both phases, without altering any other aspects of the sampling procedure.

Figure~\ref{fig:em_lyaguide} reports the results on the synthetic two-moon dataset. Under the same accuracy criterion, LyaGuide consistently requires substantially fewer integration steps and achieves a noticeable reduction in total sampling time compared to the EM baseline. Moreover, when fixing the transition time parameter $\tau^\ast$ in phase 1
 to values relatively far from~1 (we choose 0.7 here), LyaGuide leads to trajectories that remain significantly closer to the target distribution, as reflected by faster decay of the Wasserstein-2 distance. These results demonstrate that incorporating LyaGuide into EM yields more contractive dynamics, accelerates convergence toward high-density regions, and improves sampling efficiency while preserving the original EM design.

\begin{figure}[htb]
% \vskip -0.17in
	\centering
	\includegraphics[width=0.5\textwidth]{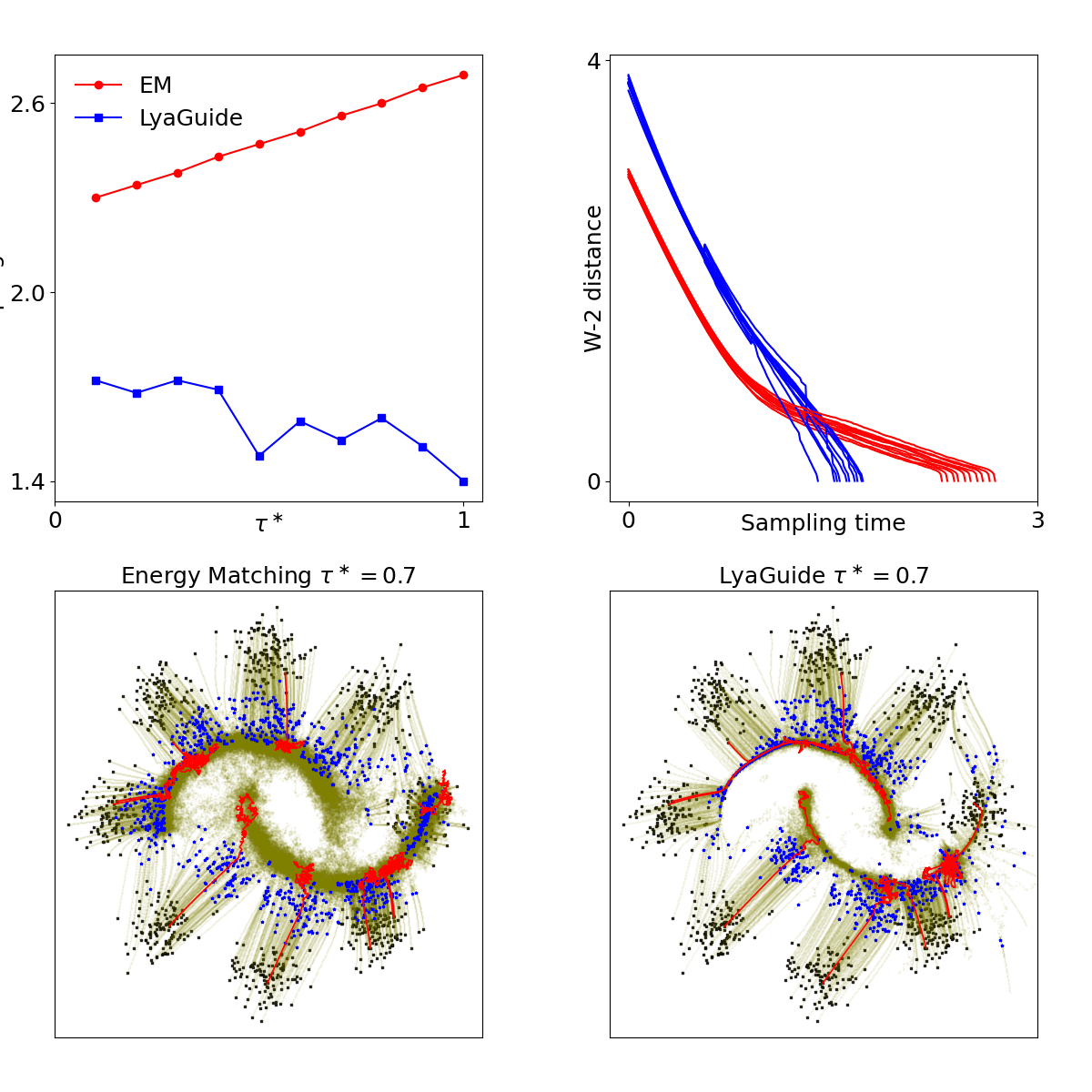}
	\caption{Upper left: Sampling time under different phase 1 duration $\tau^\ast$. Upper right: Wasserstein-2 distance between the sampled data and the converged data along the sampling trajectories, blue and red curves correspond to 10 trajectories under different $\tau^\ast$. Bottom: Diagram of the sampling process, black dots are initial samples, olive dots are converged samples, blue dots are samples at time $\tau^\ast=0.7$, red curves are the sampling trajectories.
     }
     \label{fig:em_lyaguide}
% \vskip -0.25in
\end{figure}

\section{Conclusion and Discussion}

In this work, we introduce LyaGuide, a unified framework that reformulates guidance in flow matching as a Lyapunov control problem. By establishing a theoretical equivalence between guided flows and Lyapunov control law, we show that diverse guidance strategies can be interpreted within a single control-theoretic perspective. To rigorously enforce stability, we also propose a pseudo projection operator with a closed-form expression, which guarantees that any candidate guidance satisfies the Lyapunov condition. Notably, the projection is simple to integrate into existing guidance pipelines and introduces minimal implementation overhead. This work opens new opportunities for control-inspired generative modelling, particularly in domains where explicit priors are scarce but limited supervision can be leveraged, making the framework especially suitable for deploying pre-trained models on individual systems.

\paragraph{Both stable and generating the guided probability path}
The results in Theorem~\ref{main thm} characterize the existence of a control/guidance that simultaneously satisfies: (i) the weighted continuity equation (Eq.~\eqref{eq1 in appen} in Appendix) that generates the guided probability path, (we denote such controllers form a space $\mathcal{U}_g$) and (ii) the local Lyapunov decrease condition (such controllers form a space $\mathcal{U}_s$) . The pseudo-projection in Theorem~\ref{thm st guarantee} is designed exclusively to enforce the Lyapunov inequality. 
Since the projection is a pointwise modification of the vector field, it does not impose any constraint on the weighted divergence term 
$\nabla\!\cdot(p_t'\,u_t)$, and thus does not preserve the weighted continuity equation in general. Our idea is to project a guidance term $c\in\mathcal{U}_g$ into the subspace $\mathcal{U}_g \cap \mathcal{U}_s$ to enforce stability for the guidance term, thereby accelerating the sampling process and enhancing the sampling robustness.
How to develop a projection operator or an alternative correction that enforces \emph{both} Lyapunov stability and the weighted divergence constraint remains an open problem and a promising direction for future work.

\paragraph{Fixed-Time Lyapunov Control}~Our framework currently relies on asymptotic convergence with exponential stability  instead of finite convergence time. Since generative models operate with fixed inference time, a natural question is whether fixed-time Lyapunov control can be incorporated to regulate convergence within this horizon. While classical fixed-time theory~\cite{polyakov2012nonlinear, polyakov2015finite} guarantees bounded-time stability, its direct use may neglect the data distribution $p_1(\vx)=q(\vx)$ and disrupt the Lipschitz continuity of neural ODEs. Future work could therefore explore fixed-time Lyapunov methods or design distribution-level Lyapunov functionals that couple $q(x)$ with the task-specific energy $J(x)$, enabling faithful generation with controlled convergence speed.

\paragraph{On the Role of the Potential $V$}
Our guarantees hinge on the availability or learnability of a potential $V$ that meaningfully encodes task priors. 
When $V$ is poorly specified or underfitted, guidance may bias samples toward suboptimal regions, a limitation shared with classifier/EBM/reward guidance in diffusion~\cite{dhariwal2021diffusion,ho2021classifier,janner2022planning}. Scenario~2 mitigates this by learning $V_\theta$ from few-shot supervision with a Lyapunov penalty, but the expressiveness–regularity trade-off remains: overly flexible $V_\theta$ may violate smoothness or curvature required for stabilization and overly rigid $V_\theta$ may underfit complex priors, how to address such issue in learning Lyapunov function remains unknown.

\section*{Acknowledgments}
We thank Ruiqi Feng for comments on an earlier version of this manuscript.

% \bibliographystyle{IEEEtran}
% \bibliography{iclr2025_conference}

% \clearpage
% \newpage
% \section*{Appendix}
\appendices
\section{Algorithms}

\begin{algorithm}[htb]
\caption{LyaGuide: Training-free Approach}
\label{algo1}
\begin{algorithmic}[1]
\STATE {\bfseries Input:} Learned vector field $u_t$ in flow matching, explicit Lyapunov function $V(x)$ for prior knowledge, guidance parameter $\delta$, initial parameter set $k$.
    \STATE {\bfseries Candidate guidance}: E.g., $\vc_t(\vx)=-k\nabla V(\vx)$  \hfill {$\triangleright$ user-defined}
    \STATE {\bfseries Pseudo projection operation} for rigorous satisfaction of Lyapunov condition~\ref{lya cond}:
 \STATE {\bfseries LyaGuide-ES:}~ $\vc_t^\ast(\vx)=\vc_t(\vx)-\frac{\max\left(0,\nabla V\cdot\left(\vu_t+\vc_t\right)+\delta V\right)}{\Vert\nabla V\Vert^2}\cdot\nabla V$
           \STATE {\bfseries LyaGuide-AS:}~ $\vc_t^\ast(\vx)=\vc_t(\vx)-\frac{\max\left(0,\nabla V\cdot\left(\vu_t+\vc_t\right)\right)}{\Vert\nabla V\Vert^2}\cdot\nabla V$
\STATE {\bfseries LyaGuide-CS:}~$\vc_t^\ast(\vx)_i=(\vc_{\vtheta_\vc}(\vx,t))_i-\frac{\max\left(0,(\nabla V)_i\left(\vu+\vc\right)_i+\delta V\right)}{\Vert\nabla V\Vert^2}\cdot\nabla V$

\STATE {\bfseries Output:} Guided vector field $\vf_t(x) = \vu_t(\vx) + \vc_t^\ast(\vx)$ for conditional flow matching.
\end{algorithmic}
\end{algorithm}

The score-based candidate guidance in Algorithm~\ref{algo1} provides an efficient approximation to the ideal guidance, although its performance depends on the choice of the initial parameter $k$. A better candidate may be obtained by iteratively running the algorithm with different values of $k$ and selecting the best outcome. More generally, any other method for generating a guidance term can be seamlessly integrated into our framework~\cite{fengguidance}. In particular, one may train a neural network to approximate the candidate guidance, as we demonstrate in Scenario~2 below.

 \begin{algorithm}[htb]
 	\centering
 	\caption{ \footnotesize{LyaGuide: Training-based Approach} } \label{algo2}
 	\begin{algorithmic}
 		\STATE {\bfseries Input:} Preference data-score pairs $\mathcal{D}=\{\left({\vx}_i,V_i\right)\}_{i=1}^{n}$, learning rate $\beta$,   initial parameters $\vtheta_0$, $\vtheta=(\vtheta_V,\vtheta_\vc)$, learned flow $\vu_t$ from take noisy distribution $p_0$ to data distribution $p_1$, guidance parameter $\delta$. 

        % iteration step $m$, training error $\delta$,

 	\WHILE{\textit{not converged}}
        \STATE $(\vx_i,V_i)\sim \mathcal{D}$ \hfill {$\triangleright$ sample data and score}
        \STATE Compute Lyapunov Loss $L(\vtheta_V)$ from~\eqref{eq sce 2 loss} 
        \STATE $\vtheta_V\gets\vtheta_V-\beta\nabla_\vtheta L(\vtheta_V)$ \hfill {$\triangleright$ Learn Lyapunov function}
        \STATE Train $\vtheta_\vc$ with any existing learning-based guidance method
          \ENDWHILE

             \STATE {\bfseries Pseudo projection operation :} \\
            \STATE {\bfseries LyaGuide-ES:}~$\vc_t^\ast(\vx)=\vc_{\vtheta_\vc}(\vx,t)-\frac{\max\left(0,\nabla V_{\vtheta_V}\cdot\left(\vu+\vc\right)+\delta V_{\vtheta_V}\right)}{\Vert\nabla V_{\vtheta_V}\Vert^2}\cdot\nabla V_{\vtheta_V}$
           \STATE {\bfseries LyaGuide-AS:}~$\vc_t^\ast(\vx)=\vc_{\vtheta_\vc}(\vx,t)-\frac{\max\left(0,\nabla V_{\vtheta_V}\cdot\left(\vu+\vc\right)\right)}{\Vert\nabla V_{\vtheta_V}\Vert^2}\cdot\nabla V_{\vtheta_V}$
\STATE {\bfseries LyaGuide-CS:}~$\vc_t^\ast(\vx)_i=(\vc_{\vtheta_\vc}(\vx,t))_i-\frac{\max\left(0,(\nabla V_{\vtheta_V})_i\left(\vu+\vc\right)_i+\delta V_{\vtheta_V}\right)}{\Vert\nabla V_{\vtheta_V}\Vert^2}\cdot\nabla V_{\vtheta_V}$
   \STATE {\bfseries Output:} Guided vector field $\vf_t(x) = \vu_t(\vx) + \vc_t^\ast(\vx)$ for conditional flow matching.
 	\end{algorithmic}
 \end{algorithm}%

\section{Proofs}\label{appen proof}
\subsection{Notations}~Denote by $\Vert \cdot \Vert$ the $L^2$-norm for any given vector in $\mathbb{R}^d$.  Denote by $\Vert \cdot\Vert_{C(\mathcal{D})}$ the maximum norm on continuous function space $C(\mathcal{D})$. For $A=(a_{ij})$, a matrix of dimension $d\times r$, denote by $\|A\|^2_{\rm F}=\sum_{i=1}^{d}\sum_{j=1}^{r}a_{ij}^2$ the Frobenius norm. Denote $\max(a,0)$ by $(a)^+$. Denote $\vx\cdot\vy$ as the inner product of two vectors. 

\subsection{Proof of Theorem~\ref{main thm}}\label{proof main}

\begin{theorem}[Equivalence between Guided Flow Matching and Lyapunov Control]
      For the flow model $\vu_t(\vx_t)$ that generates the probability path $p_t(\vx)$, finding the guidance $\vc_t(\vx_t)$ to the vector field $\vu_t(\vx_t)$ to perform conditional sampling $p_t'(\vx)=\frac{1}{Z_t}p_t(\vx)\mathrm{e}^{-J(\vx)}$ is equivalent to finding the controller that satisfies the local Lyapunov condition, where energy function $J(\vx)$ contributes to Lyapunov function as $V\propto (J+c)$ for some constant $c$, e.g., $V=J$. More specifically, the equivalence means that there exists a control that both generates the guided probability path and satisfies the local Lyapunov stability condition, thereby accelerating the sampling dynamics of guided flows.
\end{theorem}

\textbf{Proof.}
Let $p_t$ be the probability path governed by the continuity equation
\begin{equation}\label{eq:ce}
\partial_t p_t(\vx)+\nabla\!\cdot\!\big(p_t(\vx)\,\vu_t(\vx)\big)=0,
\end{equation}
where $\vu_t$ is the trained flow model. For energy $J(\vx)$, define the reweighted path as
\begin{equation}\label{eq:ptprime-def}
p'_t(\vx)=\frac{1}{Z_t}\,e^{-J(\vx)}\,p_t(\vx),\qquad
Z_t=\int e^{-J(\vx)}\,p_t(\vx)\,\mathrm{d}\vx.
\end{equation}
We assume sufficient smoothness, and either periodic boundary conditions or vanishing flux at infinity so that boundary integrals of divergences vanish.

\medskip
\noindent\textbf{Step 1.}
By definition \eqref{eq:ptprime-def} we have,
\begin{align}
\partial_t p'_t
&= \partial_t\!\Big(\frac{e^{-J}}{Z_t} p_t\Big)
 = \frac{e^{-J}}{Z_t}\,\partial_t p_t
   + p_t\,\partial_t\!\Big(\frac{e^{-J}}{Z_t}\Big) \nonumber\\
&= \frac{e^{-J}}{Z_t}\,\partial_t p_t
   - \big(\partial_t\log Z_t\big)\,\underbrace{\frac{e^{-J}}{Z_t}p_t}_{=\,p'_t}
\quad(\text{$J$ independent of $t$}) \nonumber\\
&= \frac{e^{-J}}{Z_t}\,\partial_t p_t - \big(\partial_t\log Z_t\big)\,p'_t.
\label{eq:dt-ptprime}
\end{align}
Substituting the continuity equation \eqref{eq:ce} into \eqref{eq:dt-ptprime} yields
\begin{equation}\label{eq:ptprime-with-ce}
\partial_t p'_t
= -\,\frac{e^{-J}}{Z_t}\,\nabla\!\cdot(p_t\vu_t) - \big(\partial_t\log Z_t\big)\,p'_t.
\end{equation}
On the other hand, for the transport term,
\begin{align}
\nabla\!\cdot\!\big(p'_t\vu_t\big)
&= \nabla\!\cdot\!\Big(\frac{e^{-J}}{Z_t}p_t\,\vu_t\Big)
 = \frac{e^{-J}}{Z_t}\,\nabla\!\cdot(p_t\vu_t)
  \\&\qquad + (p_t\vu_t)\!\cdot\nabla\!\Big(\frac{e^{-J}}{Z_t}\Big) \nonumber\\
&= \frac{e^{-J}}{Z_t}\,\nabla\!\cdot(p_t\vu_t)
   - \frac{e^{-J}}{Z_t}\,p_t\,\vu_t\!\cdot\nabla J\\
& = \frac{e^{-J}}{Z_t}\,\nabla\!\cdot(p_t\vu_t)
   - p'_t\,\vu_t\!\cdot\nabla J.
\label{eq:div-ptprime-ut}
\end{align}
Adding \eqref{eq:ptprime-with-ce} and \eqref{eq:div-ptprime-ut} gives
\begin{equation}\label{eq:key-balance}
\partial_t p'_t + \nabla\!\cdot\!\big(p'_t\vu_t\big)
= -\,p'_t\,\vu_t\!\cdot\nabla J \;-\; \big(\partial_t\log Z_t\big)\,p'_t.
\end{equation}
If we want $p'_t$ to be generated by the controlled field $\vu_t+\vc_t$, i.e.
\[
\partial_t p'_t + \nabla\!\cdot\!\big(p'_t(\vu_t+\vc_t)\big)=0,
\]
then comparing with \eqref{eq:key-balance} we obtain the \emph{weighted divergence equation} for $\vc_t$:
\begin{equation}\label{eq1 in appen}
\nabla\!\cdot\!\big(p'_t\,\vc_t\big)
= p'_t\Big(\vu_t\!\cdot\nabla J + \partial_t\log Z_t\Big).
\end{equation}

Integrating \eqref{eq1 in appen} over $\mathbb{R}^d$ and using divergence theorem,
\[
\int_{\mathbb{R}^d}\nabla\!\cdot\!\big(p'_t\vc_t\big)\,\mathrm{d}\vx = 0
\Longleftrightarrow
\int_{\mathbb{R}^d} p'_t\Big(\vu_t\!\cdot\nabla J + \partial_t\log Z_t\Big)\mathrm{d}\vx=0.
\]
The identity holds because
\begin{align*}
\partial_t\log Z_t
&= \frac{1}{Z_t}\partial_t\!\int e^{-J}p_t\,\mathrm{d}\vx\\
& = \frac{1}{Z_t}\!\int e^{-J}\partial_t p_t\,\mathrm{d}\vx\\
 &= -\frac{1}{Z_t}\!\int e^{-J}\nabla\!\cdot(p_t\vu_t)\,\mathrm{d}\vx\\
&= -\frac{1}{Z_t}\!\int \nabla\!\cdot\!\big(e^{-J}p_t\vu_t\big)\,\mathrm{d}\vx\\
&\qquad\qquad   + \frac{1}{Z_t}\!\int (p_t\vu_t)\!\cdot\nabla(e^{-J})\,\mathrm{d}\vx\\
&= \frac{1}{Z_t}\!\int (p_t\vu_t)\!\cdot\!\big(-e^{-J}\nabla J\big)\,\mathrm{d}\vx\\
& = -\!\int p'_t\,\vu_t\!\cdot\nabla J\,\mathrm{d}\vx,
\end{align*}
where the divergence integral vanishes by the boundary assumption. Hence the right-hand side of \eqref{eq1 in appen} has zero integral and \eqref{eq1 in appen} admits solutions $\vc_t$ (e.g. $\vc_t=\nabla\psi_t$ with a weighted Neumann problem for $\psi_t$).

\noindent\textbf{Step 2: Constructing a Lyapunov-compatible control.}
We claim that a solution $\vc_t$ to \eqref{eq1 in appen} can be chosen so that $J$ relates to a local Lyapunov function as $V=-J$ under the controlled flow $\vu_t+\vc_t$.  Decompose
\begin{align}
    \vc_t(\vx)&=\vc_t^{\perp}(\vx)+\vc_t^{\top}(\vx),\\
\vc_t^{\perp}(\vx)&=\alpha_t(\vx)\,\frac{\nabla V(\vx)}{\|\nabla V(\vx)\|},~
\nabla V(\vx)\!\cdot \vc_t^{\top}(\vx)\equiv 0,
\end{align}
with the convention $\alpha_t(\vx)=0$ at critical points $\{\vx:\nabla V(\vx)=0\}$. Choose
\[
\alpha_t(\vx)
= -\frac{\vu_t(\vx)\!\cdot\nabla V(\vx)}{\|\nabla V(\vx)\|}
  - \delta_t(\vx)\,\frac{V}{\|\nabla V(\vx)\|},~ \delta_t(\vx)\ge 0.
\]
Then along any trajectory $\vx_t$ driven by $\vu_t+\vc_t$,
\[
\frac{\mathrm{d}}{\mathrm{d}t}V(\vx_t)
= \nabla V(\vx_t)\!\cdot\!\big(\vu_t(\vx_t)+\vc_t(\vx_t)\big)
= -\,\gamma_t(\vx_t)V(\vx_t) ,
\]
so $V$ is locally Lyapunov. The tangential component $\vc_t^{\top}$ is then chosen to satisfy the residual of \eqref{eq1 in appen}:
\[
\nabla\!\cdot\!\big(p'_t\vc_t^{\top}\big)
= p'_t\big(\vu_t\!\cdot\nabla J + \partial_t\log Z_t\big)
  - \nabla\!\cdot\!\big(p'_t\vc_t^{\perp}\big),
\]
which admits solutions $\vc^\top$ because both sides have zero integral.

\noindent\textbf{Step 3: From Lyapunov control to guided control.}
Conversely, suppose there exists a control $\vc_t$ with
\[
V=J, \nabla V(\vx)\!\cdot\!\big(\vu_t(\vx)+\vc_t(\vx)\big)\le -
\delta V.
\]
Define
\[
\phi_t(\vx):=\partial_t\log Z_t+\vu_t(\vx)\!\cdot\nabla V(\vx),
\]
and the guided control $\tilde\vc$ should be a solution of
\[
\nabla\!\cdot\!\big(p'_t\tilde\vc_t\big)=p'_t\,\phi_t.
\]
To enforce the Lyapunov control can also satisfy the divergence equation, we correct the Lyapunov control as $\tilde{\vc}=\vc+\vw$, where $\vw$ obeys to the following equations:
\begin{equation}\label{eq appen proof 1}
    \begin{aligned}
        \nabla V(\vx)\cdot\vw_t(\vx)&=0,~\forall\vx\in O(\vx^\ast,\varepsilon),~\forall~\text{local minima~}\vx^\ast,
    \end{aligned}
\end{equation}
\begin{equation}\label{eq appen proof 2}
    \begin{aligned}
        \nabla\cdot (p_t'\vw_t)&=p_t'\phi_t-\nabla\cdot(p_t'\vc_t).
    \end{aligned}
\end{equation}
By Step~1, equation~\eqref{eq appen proof 2} is solvable under the mild boundary condition 
$\int_{\mathbb{R}^d}\nabla\!\cdot (p_t' \vw_t)\,\mathrm{d}\vx = 0$.  
The purpose of equation~\eqref{eq appen proof 1} is to ensure that the modified control 
$\tilde\vc_t$ satisfies the Lyapunov decrease condition.  
Since \eqref{eq appen proof 1} is only required to hold in a small neighborhood of the local minimum of $V$, 
there remains substantial freedom to adjust $\vw_t$ by adding any element of the nullspace of the weighted 
divergence operator $L_t(\cdot)=\nabla\!\cdot(p_t'\cdot)$.

Steps~2 and~3 provide two opposite directions for constructing a control (or guidance term) 
that simultaneously satisfies both the Lyapunov inequality and the divergence constraint~\eqref{eq1 in appen}.  
However, the proof reveals that starting from a Lyapunov-stable control and attempting to enforce the divergence equation 
requires solving the coupled PDEs~\eqref{eq appen proof 1}--\eqref{eq appen proof 2}, which is generally more difficult.  
In practice, it is considerably easier to begin with a guidance term that already satisfies the divergence equation~\eqref{eq1 in appen}, 
and then impose Lyapunov stability via the projection operation in Theorem~\ref{thm st guarantee}.

The above derivations can easily be applied to the case $V=kJ$ with $k>0$.

Combing the above three steps, we complete the proof.\hfill $\blacksquare$

\paragraph{Intuitive explanation of the theory}
The Theorem~\ref{main thm} reveals a useful conceptual symmetry between guidance-compatible controls (those satisfying the weighted divergence equation) and Lyapunov-stable controls (those ensuring monotone contraction of the Lyapunov function). In principle, starting from a Lyapunov-stable control, one can always construct a guidance-compatible control by solving an appropriate correction equation so that the divergence constraint~\eqref{eq1 in appen} is satisfied. Conversely, given any guidance term, one can adjust its normal component along $\nabla V$ to enforce Lyapunov decrease while leaving the divergence structure intact.

However, these two directions are not equally tractable. The conversion from Lyapunov stability to guidance requires solving coupled PDE system~\eqref{eq appen proof 1},\eqref{eq appen proof 2} that simultaneously enforces tangency to the level sets of 
$V$ and corrects the weighted divergence. This construction is mathematically valid but typically costly and impractical for high-dimensional learning problems.

The reverse direction is substantially simpler and forms the basis of LyaGuide. Modern flow-matching or score-based models already produce guidance terms that satisfy the divergence equation by construction. Given such a guidance field, Theorem~\ref{thm st guarantee} shows that we can impose Lyapunov stability without solving any PDE by projecting the field onto the Lyapunov-stable set at each point. The projection adjusts only the normal component of the guidance vector along $\nabla V$, guaranteeing strict decay of sampling trajectory along
$V$ while preserving the original divergence structure.

This two-way relationship explains why LyaGuide is both principled and efficient: the theory ensures that stable and divergence-compatible controls are mutually reachable, whereas the algorithm exploits the easy direction—projecting guidance to stability—to obtain a control that satisfies both properties simultaneously.

%\paragraph{Remark (time–dependent weight).}
%If one prefers to start strictly from $p_0$, one may use an annealed weight $p'_{t}\!\propto p_t e^{-\lambda(t)J}$ with $\lambda(0)=0$, which yields the same derivation as above with the right-hand side of \eqref{eq1 in appen} replaced by $p'_t\!\big(\dot\lambda(t)J+\lambda(t)\,\vu_t\!\cdot\nabla J+\partial_t\log Z_t\big)$.

\subsection{Proof and Analysis of Proposition~\ref{prop1}}\label{sec unifying results}

\begin{proposition}
The following commonly used guidance strategies in generative modeling can all be interpreted as Lyapunov control within our unified framework:
\begin{itemize}
    \item \textbf{Classifier Guidance}: Given a trained classifier $p(\vy|\vx)$, the Lyapunov function for guided distribution $p(\vx|\vy)$ specified on conditioner $\vy$ is $V_\vy(\vx)=-\log p(\vy|\vx)$.
    \item \textbf{Reward Guidance}: In reinforcement learning tasks with reward function $R(\vx)$, the Lyapunov function for guided distribution $\tfrac{1}{Z} p_t(\vx) e^{R(\vx)}$ concentrates probability mass in high-reward regions is $V(\vx)=-R(\vx)$.
    \item \textbf{Energy-Based Model (EBM) Guidance}: For a target EBM $p(\vx) \propto e^{-E(\vx)}$, the Lyapunov function is naturally  $V(\vx) = -E(\vx)$.
    \item \textbf{Image inverse problems.} Let the forward operator be $\vy=H(\vx)+\varepsilon$ with $\varepsilon\sim\mathcal{N}(0,\sigma^2 I)$. 
    Then $p(\vy|\vx)\propto \exp\!\big(-\tfrac{1}{2\sigma^2}\|H(\vx)-\vy\|_2^2\big)$ and a natural Lyapunov function is $
    V_{\vy}(\vx)=\frac{1}{2\sigma^2}\,\|H(\vx)-\vy\|_2^2$, 
    yielding guided sampling $p'_t(\vx)\propto p_t(\vx)\,\exp\!\big(-V_{\vy}(\vx)\big)$ that enforces data consistency~\cite{song2023pseudoinverse}.
\end{itemize}
\end{proposition}

\textbf{Proof.}
\begin{itemize}
    \item \textbf{Classifier Guidance and image inverse problems.}:  
    Taking $J(x)=-\log p(y|x)$, we compute
    \begin{align}
    \tfrac{1}{Z} p_t(x) e^{-J(x)} 
    &= \frac{p_t(x) p(y|x)}{\int p_t(x)p(y|x)\,dx}\\
    &= \frac{p_t(x) \tfrac{p(x,y)}{p(x)}}{\int p_t(x)\tfrac{p(x,y)}{p(x)}\,dx}\\
    &= p(x|y).    
    \end{align}
    Thus the guided distribution is exactly conditional sampling.
    
    \item \textbf{Reward Guidance}:  
    With $J(x)=-R(x)$, the guided law is
    \[
    p'_t(x) = \tfrac{1}{Z} p_t(x) e^{R(x)},
    \]
    which emphasizes regions of higher reward. This is equivalent to Lyapunov control with $V(x)=-R(x)$, where the control term reduces $V(x)$ along the trajectory, steering toward reward-maximizing states.
    
    \item \textbf{EBM Guidance}:  
    If the target distribution is $p(x)\propto e^{-E(x)}$, then
    \[
    p'_t(x) = \tfrac{1}{Z} p_t(x) e^{-E(x)},
    \]
    which matches the EBM density. Interpreting $V(x)=E(x)$ as Lyapunov function, the control enforces descent in $V(x)$, aligning the flow with low-energy modes.
\end{itemize}\hfill $\blacksquare$

\subsubsection{Proof of Theorem~\ref{thm st guarantee}}\label{append proof pseudo}

\begin{theorem}[Lyapunov Guarantee for Guidance]
 For a candidate controller $\vc$ and the guidance controller space $\mathcal{U}(V)=\{\vu:\nabla V\cdot(\vu_t+\vc_t)+\delta V\le0\}$ that rigorously satifies the local Lyapunov condition in Proposition~\ref{lya cond}, 
define the projection operator as
% \begin{strip}
% \centering
% \noindent\rule{\textwidth}{0.4pt}
% \vspace{0.5em}
\begin{equation*}
	\resizebox{\linewidth}{!}{$
\pi(\vc_t,\mathcal{U}(V)) \triangleq 
\vc_t- \frac{\max\bigl(0,\nabla V(\vx)\cdot(\vu_t(\vx)+\vc_t(\vx))+\delta V(\vx)\bigr)}{\|\nabla V(\vx)\|^2}\,\nabla V(\vx).
$}
\end{equation*}
% \vspace{0.2em}
% \noindent\rule{\textwidth}{0.4pt}
% \end{strip}
Then ${\pi}(\vc_t,\mathcal{U}(V))$ is locally Lipschitz continuous and thus the guided flow under ${\pi}(\vc_t,\mathcal{U}(V))$ is well defined, and ${\pi}(\vc_t,\mathcal{U}(V)) \in \mathcal{U}(V)$.
\end{theorem}

\textbf{Proof.} The local Lipschitz continuity of the projected function naturally comes from the local Lipschitz continuity of the considered functions $\vu$, $\vc$ and $V$. We directly check the inequality constraint in $\mathcal{U}(V)$ is satisfied by the projection element, that is 
\begin{equation*}\label{st to check}
	\mathcal{L}_{\vu+\vc^\ast}V\big|_{\vc^\ast={\pi}(\vc,\mathcal{U}(V))}\le \delta -V.
\end{equation*}
Since the controller has affine actuator, from the definition of the Lie derivative operator, we have
\begin{equation*}
	\resizebox{\linewidth}{!}{$
	\begin{aligned}
	\mathcal{L}_{\vu+\vc^\ast}V\big|_{\vc^\ast={\pi}(\vc,\mathcal{U}(V))}&=\nabla V\cdot(\vu+\vc-\dfrac{\max(0,\mathcal{L}_{\vu+\vc}V+\delta V)}{\Vert \nabla V\Vert^2}\cdot\nabla V)\\
	&=\nabla V\cdot(\vu+\vc)-\nabla V\cdot \dfrac{\max(0,\mathcal{L}_{\vu+\vc}V+\delta V)}{\Vert \nabla V\Vert^2}\cdot\nabla V\\
	&=\mathcal{L}_{\vu+\vc}V-\max(0,\mathcal{L}_{\vu+\vc}V+\delta V)\le -\delta V.
	\end{aligned}
    $}
\end{equation*} \hfill $\blacksquare$

\begin{remark}
To avoid ambiguity, we distinguish two different notions that are sometimes both informally referred to as ``local stability.'' In the classical control-theoretic sense, local stability describes the situation where a given equilibrium is stable only within a neighborhood, in contrast to global stability where the same equilibrium attracts all trajectories. In the present work, however, the more relevant distinction is between a single-equilibrium setting and a multi-equilibrium (or multi-modal) setting. 

Generative modeling typically belongs to the latter case: the target distribution often contains multiple modes, and the corresponding potential $V$ admits multiple local minimizers. In such a setting, the appropriate objective is not to enforce a single globally attracting equilibrium, which would collapse all modes, but rather to ensure that each minimizer is locally stable within its own basin of attraction. Accordingly, LyaGuide uses the Lyapunov function to induce local contraction around the appropriate mode once the trajectory enters its basin, rather than to solve a single global optimization problem.

When the target distribution has only one minimizer, the same pseudo-projection mechanism reduces to the standard single-equilibrium case and yields global stability. Therefore, the use of a local Lyapunov function here reflects the multi-modal structure of the generative problem, rather than a limitation of the framework. Moreover, as illustrated in Fig.~\ref{fig_lya_sample steps}, such local contraction also accelerates convergence and reduces the Wasserstein-2 distance throughout inference.
\end{remark}

\subsection{Notes on robustness of the initial value}
In this section, we discuss the influence of the initial value to the guided generative modelling. We note that in the original flow model, the initial data is sampled from the initial distribution $p_0$. For the guided flow, we still sample the initial value from $p_0$, which is common in flow matching community~\cite{lipman2023flow,tong2023conditional,fengguidance}. However, according to the continuity equation, to generate $p_t'$ in Theorem~\ref{main thm}, the initial distribution should be $p_0'=1/Z_0e^{-J(\vx)}p_0$ instead of $p_0$.  So we investigate how the choice of initial distribution influence the generation effect from the perspective of distribution convergence.

\begin{theorem}[Robustness to Initial Distribution]
Let $u_t(x)$ be the learned vector field and $c_t(x)$ the guidance control derived from Lyapunov function $V(x)$. Assume the guided field $f_t(x) = u_t(x) + c_t(x)$ satisfies the local Lyapunov condition with rate $\delta>0$. Then, for any two initial distributions $p_0$ and $p_0^\star$, the corresponding marginals $p_t$ and $p_t^\star$ along the flow satisfy
\[
W_2(p_t, p_t^\star) \leq e^{-\delta t}\, W_2(p_0, p_0^\star), \quad \forall t \in [0,1].
\]
In particular, even if the actual inference starts from a different initial law $p_0$ (e.g.\ Gaussian noise) instead of the theoretical $p_0^\star$, the terminal distribution at $t=1$ remains exponentially close to the desired guided distribution $p_1^\star \propto q(x)e^{-J(x)}$.
\end{theorem}

\begin{lemma}[Lyapunov-guided field implies contraction]\label{lem:lyap-to-contract}
Let $f_t(x)=u_t(x)+c_t(x)$ with a Lyapunov-based guidance $c_t(x)=-K_t(x)\nabla V(x)$, where $V\in\mathcal C^2$ and $K_t(x)\in\mathbb R^{d\times d}$ is symmetric positive definite.
Fix a forward-invariant sublevel set $\mathcal S_\rho:=\{x:V(x)\le\rho\}$.
Assume on $\mathcal S_\rho$:
\begin{enumerate}
\item \textbf{Strong convexity of $V$:} $\ \nabla^2 V(x)\succeq m I$ for some $m>0$;
\item \textbf{Uniform gain:} $\ K_t(x)\succeq \kappa I$ for some $\kappa>0$;
\item \textbf{Bounded symmetric Jacobian of $u_t$:} $\ \tfrac12\!\big(\nabla u_t(x)+\nabla u_t(x)^\top\big)\preceq L I$ for some $L\in\mathbb R$;
\item \textbf{(Optional) Slowly-varying gain:} $\ \|\nabla K_t(x)\|\le B$ (if $K_t$ depends on $x$).
\end{enumerate}
Then, for all $x,y\in\mathcal S_\rho$ and $t\in[0,1]$,
\[
\big\langle x-y,\ f_t(x)-f_t(y)\big\rangle
\ \le\ -\delta\,\|x-y\|^2,
\]
with
\begin{align}
    \delta&=\kappa m - L - \varepsilon,\\
\varepsilon&=\begin{cases}
0, ~ \text{if }K_t\text{ is constant in }x;\\[2pt]
B\,\sup_{z\in\mathcal S_\rho}\|\nabla V(z)\|, ~ \text{otherwise}.
\end{cases}
\end{align}
In particular, if $\kappa m > L+\varepsilon$, the contraction condition \eqref{eq:contract} holds on $\mathcal S_\rho$.
\end{lemma}

\textbf{Proof.}
By the mean-value integral along the segment $\gamma(\theta)=y+\theta(x-y)$,
\[
f_t(x)-f_t(y)=\left(\int_0^1 \nabla f_t(\gamma(\theta))\,d\theta\right)(x-y).
\]
Hence
\[
\langle x-y,\ f_t(x)-f_t(y)\rangle
=\int_0^1 \big\langle x-y,\ \mathrm{sym}\,\nabla f_t(\gamma(\theta))\,(x-y)\big\rangle d\theta,
\]
where $\mathrm{sym}\,A=\tfrac12(A+A^\top)$. It suffices to upper bound $\mathrm{sym}\,\nabla f_t$.
Compute
\[
\nabla f_t
=\nabla u_t-\nabla\!\big(K_t\nabla V\big)
=\nabla u_t-\big[(\nabla K_t)\nabla V^\top+K_t\nabla^2 V\big].
\]
Taking symmetric parts yields
\[
\mathrm{sym}\,\nabla f_t
\ \preceq\ L I\ -\ K_t\nabla^2 V\ +\ \mathrm{sym}\big((\nabla K_t)\nabla V^\top\big).
\]
On $\mathcal S_\rho$, by (A1)–(A3): $K_t\nabla^2 V\succeq \kappa m I$ and $\mathrm{sym}((\nabla K_t)\nabla V^\top)$ has operator norm $\le \|\nabla K_t\|\,\|\nabla V\|\le B\,\sup_{\mathcal S_\rho}\|\nabla V\|$. Therefore
\[
\mathrm{sym}\,\nabla f_t\ \preceq\ \big(L-\kappa m+\varepsilon\big) I.
\]
Integrating along $\gamma(\theta)$ gives
$\langle x-y,f_t(x)-f_t(y)\rangle\le -( \kappa m-L-\varepsilon)\|x-y\|^2$,
which proves the claim with $\delta=\kappa m-L-\varepsilon$.
\hfill $\blacksquare$

\begin{theorem}[Robustness to the Initial Distribution]
Assume $c_t=-K_t\nabla V$ and conditions (A1)–(A4) in Lemma~\ref{lem:lyap-to-contract} hold on a forward-invariant sublevel set; then \eqref{eq:contract} follows with $\delta=\kappa m-L-\varepsilon$.

Let $u_t$ be the learned vector field and $c_t$ the guidance control derived from a Lyapunov function $V$, and denote the guided field by $f_t := u_t + c_t$. 
Assume there exists $\delta>0$ and a domain $\mathcal D\subset\mathbb R^d$ that is forward invariant under $f_t$ such that
\begin{equation}\label{eq:contract}
\big\langle x-y,\; f_t(x)-f_t(y)\big\rangle \;\le\; -\,\delta\,\|x-y\|^2,\qquad
\forall\,x,y\in\mathcal D,\ \forall t\in[0,1].
\end{equation}
Let $p_t$ and $p_t^\star$ be the marginals obtained by pushing forward $p_0$ and $p_0^\star$ along the flow of $f_t$, respectively. Then, for all $t\in[0,1]$,
\[
W_2\!\left(p_t,p_t^\star\right)\ \le\ e^{-\delta t}\,W_2\!\left(p_0,p_0^\star\right).
\]
In particular, taking $p_0^\star$ as the ``theoretical'' initialization (e.g.\ the marginal induced by the conditional path), any inference initialized from a different prior $p_0$ (e.g.\ a Gaussian) remains exponentially close to the target guided law at $t=1$.
\end{theorem}

\textbf{Proof.}
Let $\Phi_{t,0}:\mathcal D\to\mathcal D$ denote the flow map associated with the ODE
$\dot x = f_t(x)$, i.e., $x_t=\Phi_{t,0}(x_0)$. 
Fix any two initial states $x_0,y_0\in\mathcal D$ and consider the distance
$D(t):=\tfrac12\|x_t-y_t\|^2$ where $x_t=\Phi_{t,0}(x_0)$ and $y_t=\Phi_{t,0}(y_0)$.
Then
\[
\dot D(t)
= \big\langle x_t-y_t,\; f_t(x_t)-f_t(y_t)\big\rangle
\ \le\ -\,\delta\,\|x_t-y_t\|^2
= -\,2\delta\,D(t),
\]
where we used \eqref{eq:contract}. By Gr\"onwall's inequality,
\[
\|x_t-y_t\|\ \le\ e^{-\delta t}\,\|x_0-y_0\|,\qquad \forall t\in[0,1].
\]
Now let $\gamma_0$ be any coupling in $\Pi(p_0,p_0^\star)$. Push it forward through the product flow to obtain a coupling 
$\gamma_t := (\Phi_{t,0}\times \Phi_{t,0})_{\#}\gamma_0 \in \Pi(p_t,p_t^\star)$.
Then
\begin{equation*}
    	\resizebox{\linewidth}{!}{$
\begin{aligned}
    \int \|x-y\|^2\, d\gamma_t(x,y)
&= \int \big\|\Phi_{t,0}(x_0)-\Phi_{t,0}(y_0)\big\|^2\, d\gamma_0(x_0,y_0)\\
 &\le\ e^{-2\delta t}\!\int \|x_0-y_0\|^2\, d\gamma_0(x_0,y_0).
\end{aligned}
$}
\end{equation*}
Taking the infimum over all initial couplings $\gamma_0\in\Pi(p_0,p_0^\star)$ yields
\[
W_2^2(p_t,p_t^\star)\ \le\ e^{-2\delta t}\,W_2^2(p_0,p_0^\star),
\]
and the stated bound follows by taking square roots. \hfill $\blacksquare$
\begin{remark}[Link to Lyapunov condition]
The contractivity assumption \eqref{eq:contract} is implied locally if the Lyapunov control is taken as 
$c_t(x)=-K_t(x)\nabla V(x)$ with $K_t(x)\succeq\kappa I$, and $V$ is $m$-strongly convex on a forward-invariant sublevel set (so that $\nabla^2 V\succeq mI$ there), while the symmetric part of $\nabla u_t$ is bounded above by $L_t I$ with $\kappa m - L_t \ge \delta>0$. 
Then $\tfrac12\big(\nabla f_t+\nabla f_t^\top\big)\preceq -\delta I$, which is equivalent to \eqref{eq:contract}.
\end{remark}

This result formalizes the robustness of guided flow matching with respect to the choice of the initial distribution. Although the theoretical formulation involves the marginal initialization $p_0^\star = \int p_0(x\mid x_1)q(x_1),dx_1$, in practice one typically samples directly from a Gaussian prior $p_0$. The contraction guaranteed by the Lyapunov condition ensures that the discrepancy between $p_0$ and $p_0^\star$ vanishes exponentially fast along the trajectory. Consequently, by the terminal time $t=1$, the generated distribution $p_1$ is already arbitrarily close to the target law, making inference from $p_0$ both valid and effective.

\subsection{Experimental Configurations and Additional Experiments}\label{appen_details}

\subsubsection{Synthetic Data Experiments.}
For the 2D synthetic benchmarks (uniform-to-8Gaussians, circle-to-S-curve and 8Gaussians-to-Moons), we follow the setup of~\cite{fengguidance}. The flow model is trained with displacement interpolation and a 4-layer MLP with hidden dimension 128, using Adam optimizer with learning rate $10^{-4}$. Each experiment is trained for 20k iterations with batch size 512. Guidance baselines include Monte Carlo ($g^{\text{MC}}$), covariance-based guidance ($g^{\text{cov-A}}$, $g^{\text{cov-G}}$), contrastive energy guidance (CEG), and learned guidance $g_\phi$. For each method, we apply our pseudo projection operation (LyaGuide-ES/AS/CS) in the inference stage.

\textbf{Few-shot Supervision (Scenario 2).} In Scenario 2, we subsample preference data–score pairs with sizes varying from 128 to 1024. The Lyapunov candidate $V_\theta$ is parameterized as an MLP with 3 hidden layers (width 64) and trained using weighted regression. We emphasize low-score samples to encourage convexity around minima. Once $V_\theta$ is learned, the corresponding guidance control is synthesized using either explicit descent or integrated into $g_\phi$.

\textbf{Ablation study.} 
In this section we investigate the effects of the parameters $\delta,k$ to the performance of LyaGuide. The parameter $\delta$ appears both in the Lyapunov convergence rate in Proposition~3.1 and in the explicit projection term in Theorem~4.1. This dual influence makes $\delta$ a global hyperparameter affecting all initial methods when incorporated into LyaGuide. From Figure~\ref{fig ablation}, we observe that the method is overall robust to $\delta$ across a broad range of values. As $\delta$ increases, the generated samples concentrate more strongly near the minima of the Lyapunov function $V$, i.e., the low-energy region. This stronger contraction improves convergence but also reduces the probability of visiting low-density regions, thereby decreasing exploration capability. In other words, a large $\delta$ induces over-contraction, which may excessively bias sampling towards high-density areas. Therefore, $\delta$ can be tuned according to practical needs: a smaller $\delta$ enhances exploration, while a larger $\delta$ yields stronger stability and contraction. However, as shown in both Figure~\ref{fig ablation} and Table~\ref{tab ablation}, excessively large $\delta$ may cause LyaGuide to underperform compared to the original guidance. Based on our empirical findings, a practical recommendation is to choose $\delta \in [0,1]$, which consistently yields a good balance between exploration and stability.

\textbf{Sensitivity with respect to $k$ (gradient-based guidance only).}
The control coefficient $k$ only affects gradient-based guidance methods (e.g., energy- or score-based guidance), since it appears in the initial guidance term $-k \nabla V$ prior to applying LyaGuide. Therefore, we evaluate $k$ exclusively on gradient-based variants, as reported in Table~\ref{tab ablation}. The results show that increasing $k$ generally improves sample quality, as stronger control in the Lyapunov direction facilitates more effective contraction of the gradient flow. This benefit persists across different $\delta$ values, although extreme $\delta$ can diminish the improvement. These observations confirm that $k$ acts as a refinement factor for stabilising gradient-based methods, rather than as a global sensitivity parameter.

\textbf{Brief Summary.}
The parameter $\delta$ influences both Lyapunov convergence and the projection step, and the method exhibits low sensitivity to its variation. Larger $\delta$ improves convergence but reduces exploration; thus, we recommend choosing $\delta \in [0,1]$. The parameter $k$ affects only gradient-based guidance, and larger $k$ generally improves performance, as shown in Table~\ref{tab ablation}.

\begin{figure*}[htb]
% \vskip -0.17in
	\centering
	\includegraphics[width=1.0\textwidth]{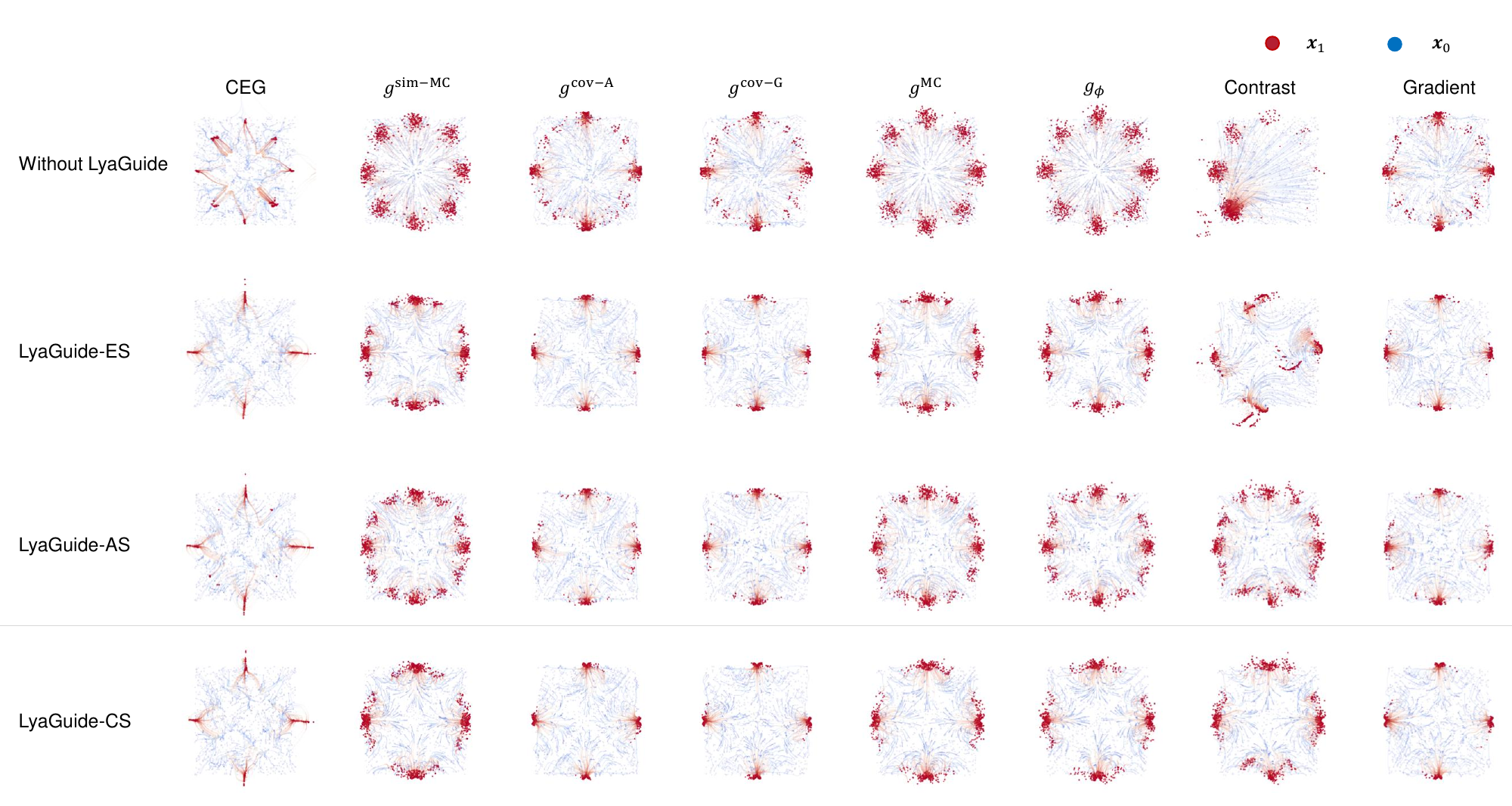}
	\caption{Scenario 1 results in 8-Gaussian task with different variants of LyaGuide. 
     }% \vskip -0.25in
\end{figure*}

\begin{figure*}[htb]
% \vskip -0.17in
	\centering
	\includegraphics[width=0.8\textwidth]{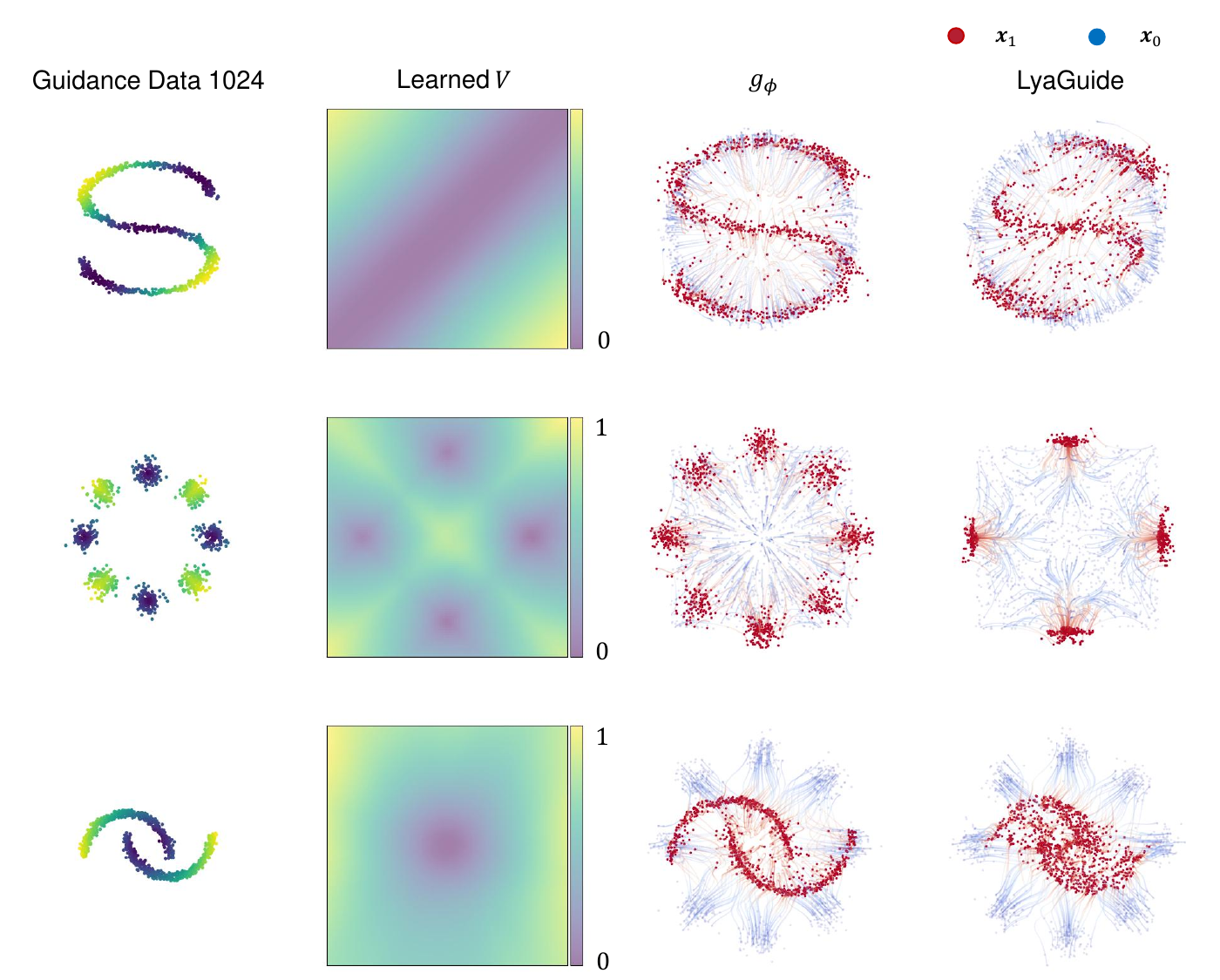}
	\caption{Scenario 2 results on synthetic data with dataset size $=1024$. 
     }
% \vskip -0.25in
\end{figure*}

\begin{figure*}
% \vskip -0.17in
	\centering
	\includegraphics[width=0.8\textwidth]{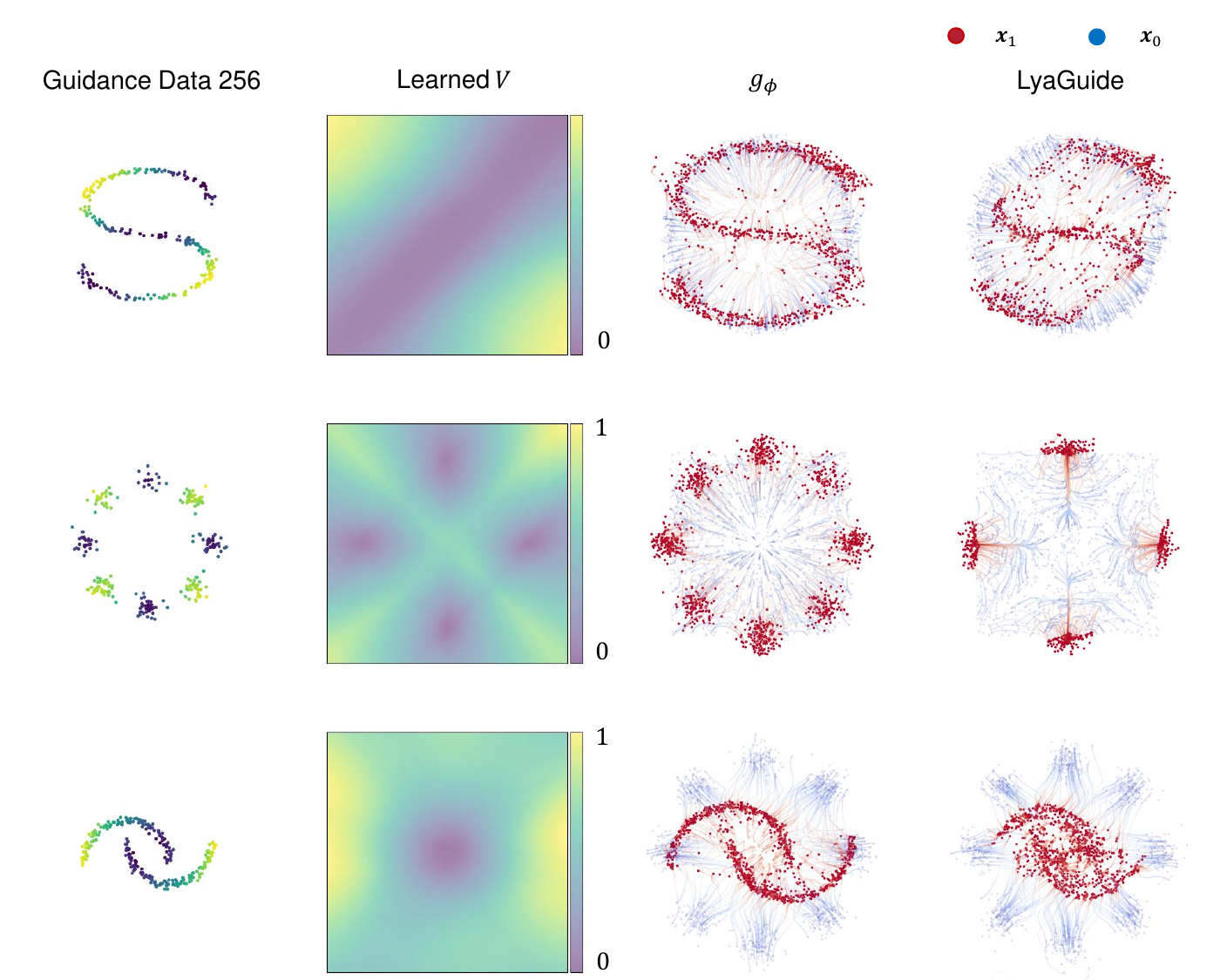}
	\caption{Scenario 2 results on synthetic data with dataset size $=256$. 
     }
% \vskip -0.25in
\end{figure*}

\begin{figure*}
% \vskip -0.17in
	\centering
	\includegraphics[width=0.8\textwidth]{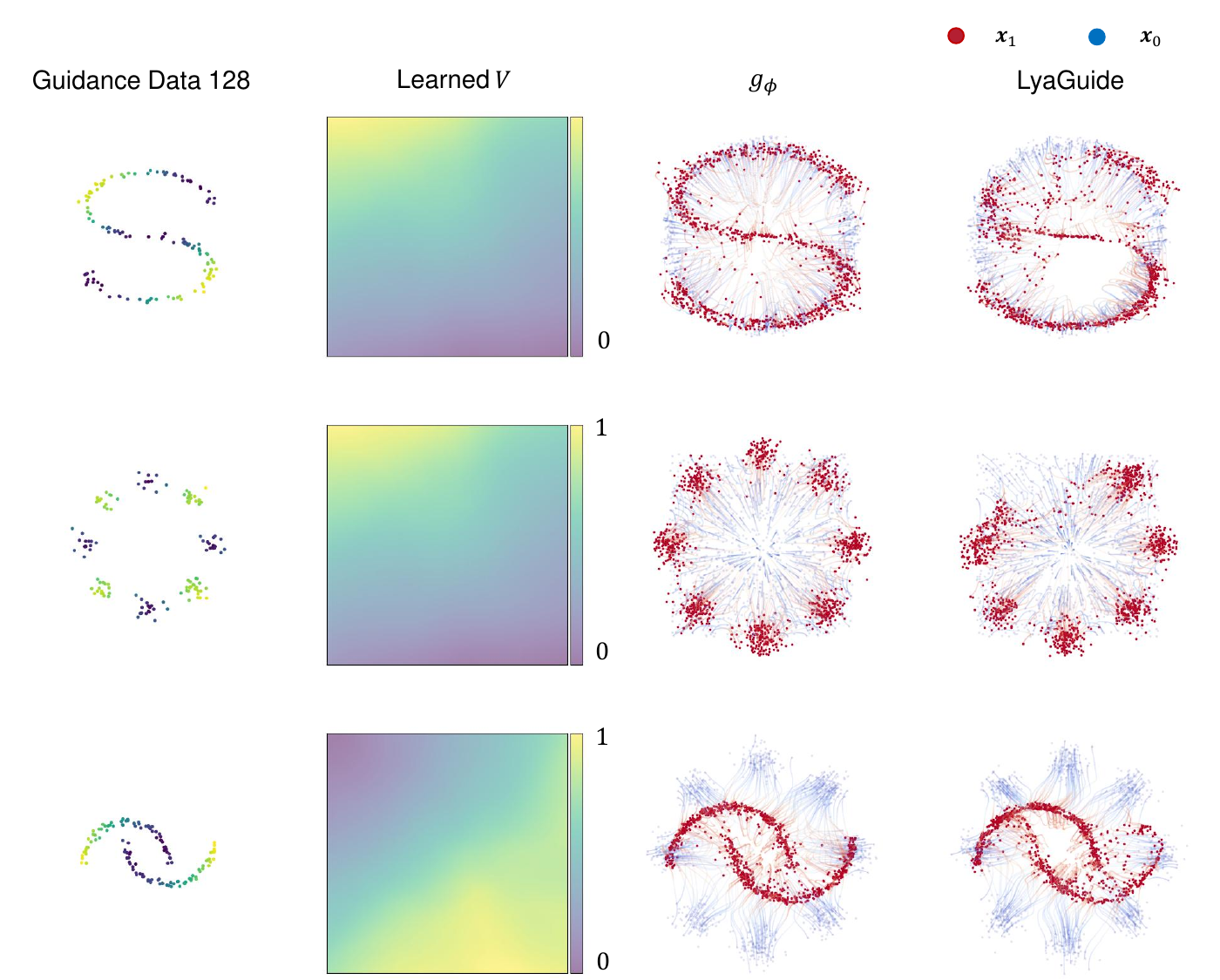}
	\caption{Scenario 2 results on synthetic data with dataset size $=128$. 
     }% \vskip -0.25in
\end{figure*}

\clearpage

\subsubsection{Image Inverse Problems}
For CelebA-HQ image inpainting, deblurring, and super-resolution tasks, we adopt the same experimental protocol as~\cite{song2023pseudoinverse, fengguidance}. Conditional flow matching (CFM) and optimal-transport CFM (OT-CFM) are used as base models. Evaluation metrics include FID, LPIPS, PSNR, and SSIM over 3000 test samples. Baselines ($g^{\text{MC}}$, $g^{\text{cov-A}}$, $\Pi$GDM, etc.) are compared with and without LyaGuide projection. Hyperparameters for the pseudo projection are fixed across all tasks, demonstrating robustness without extra tuning.

\begin{table*}[htb]
	\centering
	\caption{Image inverse problem results on CelebA-HQ (Super-Resolution task).}
	\resizebox{\linewidth}{!}{
			\begin{tabular}{cccccccccccccccc}
				\toprule
				\toprule
				\multicolumn{2}{c}{\multirow{2}{*}{}}   & \multicolumn{4}{c}{Original Methods} &
				\multicolumn{4}{c}{LyaGuide-ES} & \multicolumn{4}{c}{LyaGuide-CS}   \\
				\cmidrule(lr){3-6}  \cmidrule(lr){7-10}
				\cmidrule(lr){11-14} 
				&	&FID $\downarrow$  &LPIPS $\downarrow$ &PSNR  $\uparrow$ &SSIM $\uparrow$  &FID $\downarrow$  &LPIPS $\downarrow$ &PSNR  $\uparrow$ &SSIM $\uparrow$ &FID $\downarrow$  &LPIPS $\downarrow$ &PSNR  $\uparrow$ &SSIM $\uparrow$   \\
				\midrule 
				\multicolumn{1}{c}{\multirow{3}{*}{OT-CFM}}	&{$g^{\text{cov-A}}$}   &30.2284 	&0.3713 &	22.9642 &	0.5988 &30.2183 &	0.3712 	&22.9707 &	0.5992 &27.5309 &	0.3630 	&22.8065 &	0.6132  \\
				&{$g^{\text{sim-A}}$}  & 31.8224 	&0.3718 &	23.8667 &	0.6069 &31.5397 &0.3705 	&23.8698 &	0.6091& 27.5309 &0.3630 	&22.8065 &	0.6132  \\
				&{$\Pi$GDM}  &22.9596 &	0.2826 	&26.9492 	&0.7584 &23.0441 &	0.2827 	&26.9462 &	0.7582 &23.7278 &	0.2854 	&26.9099 &	0.7554  \\
				&{$g^{\text{MC}}$}   &24.1797 	&0.5521 	&8.7411 &	0.3628& 24.9576 &	0.5538 &	8.8167 	&0.3637 &25.3356 	&0.5430 &	9.0702 	&0.3848    \\
				\midrule
				\multicolumn{1}{c}{\multirow{5}{*}{CFM}}	&{$g^{\text{cov-A}}$} & 31.9606 	&0.3769 &	22.7715 &	0.5897 &32.0031 	&0.3769 	&22.7778 	&0.5902 & 31.2359 	&0.3741 	&22.7429 	&0.5948  \\
				
				&{$g^{\text{sim-A}}$}   &32.6209 	&0.3745 	&23.6848 	&0.6005 &31.9223 	&0.3712 &	23.6621 &	0.6061  &27.7175 &	0.3564 &	23.4357 &	0.6312  \\
				&{$\Pi$GDM} &25.7605 	&0.2900 	&26.8810 &	0.7470 &25.9728 &	0.2897 	&26.8811 &	0.7473  &26.7152 &	0.2933 	&26.8322 	&0.7434 \\
				% \cdashline{2-11}[3pt/3pt]
				&{$g^{\text{MC}}$} &26.5714 &	0.5555 	&8.9762 	&0.3588   &28.1408 &	0.5566 &	9.0434 	&0.3595 &27.7506 &	0.5456 	&9.3742 &	0.3836   \\
				\midrule
				\bottomrule
			\end{tabular}
		}
		\label{tab2 super} 
	\end{table*}

\begin{table*}[htb]
	\centering
	\caption{Image inverse problem results on CelebA-HQ (Gaussian deblurring task).}
	\resizebox{\linewidth}{!}{
		\begin{tabular}{cccccccccccccccc}
			\toprule
			\toprule
			\multicolumn{2}{c}{\multirow{2}{*}{}}   & \multicolumn{4}{c}{Original Methods} &
			\multicolumn{4}{c}{LyaGuide-ES} & \multicolumn{4}{c}{LyaGuide-CS}   \\
			\cmidrule(lr){3-6}  \cmidrule(lr){7-10}
			\cmidrule(lr){11-14} 
		&	&FID $\downarrow$  &LPIPS $\downarrow$ &PSNR  $\uparrow$ &SSIM $\uparrow$  &FID $\downarrow$  &LPIPS $\downarrow$ &PSNR  $\uparrow$ &SSIM $\uparrow$ &FID $\downarrow$  &LPIPS $\downarrow$ &PSNR  $\uparrow$ &SSIM $\uparrow$   \\
			\midrule 
		\multicolumn{1}{c}{\multirow{3}{*}{OT-CFM}}	&{$g^{\text{cov-A}}$} &15.7897 &	0.2738 	&26.0362 	&0.7202   &14.9309 	&0.2699 	&26.2131 &	0.7247 & 11.9049 &	0.2534 &	26.2578 &	0.7375 \\
        	&{$g^{\text{sim-A}}$}   &14.7355 &	0.2506 	&27.8596 &	0.7721  &14.9201 &	0.2507 	&27.8551 &	0.7720  &14.7920 &	0.2500 	&27.8420 &	0.7719  \\
            	&{$\Pi$GDM}  &20.4083 	&0.2514 	&28.6943 	&0.7827  & 19.8694 &	0.2492 	&28.7049 	&0.7838  &20.3166 &	0.2517 	&28.6720 	&0.7820  \\
                	&{$g^{\text{MC}}$}  &24.5336 	&0.5524  	&8.7690  	&0.3640   & 28.0601  &	0.5156  	&10.5469  	&0.4110   &27.5108  &	0.5388  &9.2243  	&0.3913   \\
                \midrule
		\multicolumn{1}{c}{\multirow{5}{*}{CFM}}	&{$g^{\text{cov-A}}$} &16.6399 	&0.2830 	&25.5352 	&0.6989  &16.0158 &	0.2792 	&25.6866 &	0.7031 & 13.1084 &	0.2641 	&25.7225 	&0.7161 \\
 
  &{$g^{\text{sim-A}}$} &15.2263 &	0.2591 &	27.5387 &	0.7587  & 15.2263 	&0.2591 &	27.5387 &	0.7587  & 15.6122 &	0.2585 	&27.4902 &	0.7581 \\
			&{$\Pi$GDM} &20.7786 	&0.2593 	&28.3883 	&0.7709  &20.3063 &	0.2557 &	28.5493 &	0.7750 &20.6238 &	0.2591 &	28.4349 &	0.7709  \\
                	&{$g^{\text{MC}}$}  &26.4901  	&0.5556   	&9.0386   	&0.3612    & 31.1472   &	0.5276   	&10.3195   	&0.3959    &30.0128   &	0.5410   &9.5963   	&0.3917    \\
			% \cdashline{2-11}[3pt/3pt]
		
			\midrule
			\bottomrule
		\end{tabular}
	}
	\label{tab3 deblur} 
\end{table*}

\begin{figure*}[htb]
% \vskip -0.17in
	\centering
	\includegraphics[width=1.0\textwidth]{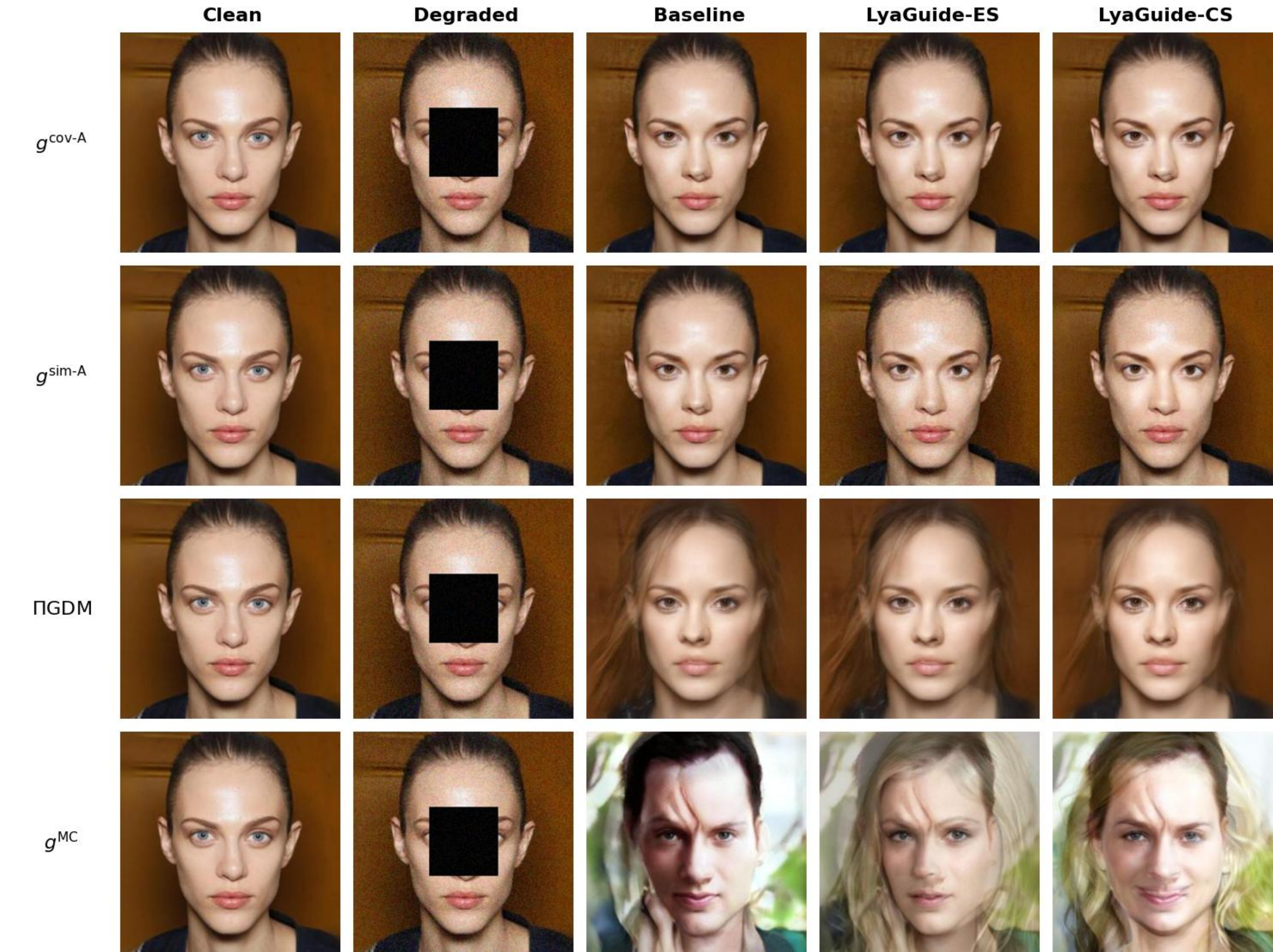}
	\caption{The visualization of the image inverse problems with the base flow matching model of mini-batch optimal
 transport conditional flow matching (OT-CFM). Four rows show the results of four baselines with the corresponding LyaGuide results in box-inpainting task.
     }
% \vskip -0.25in
\end{figure*}

\begin{figure*}
% \vskip -0.17in
	\centering
	\includegraphics[width=1.0\textwidth]{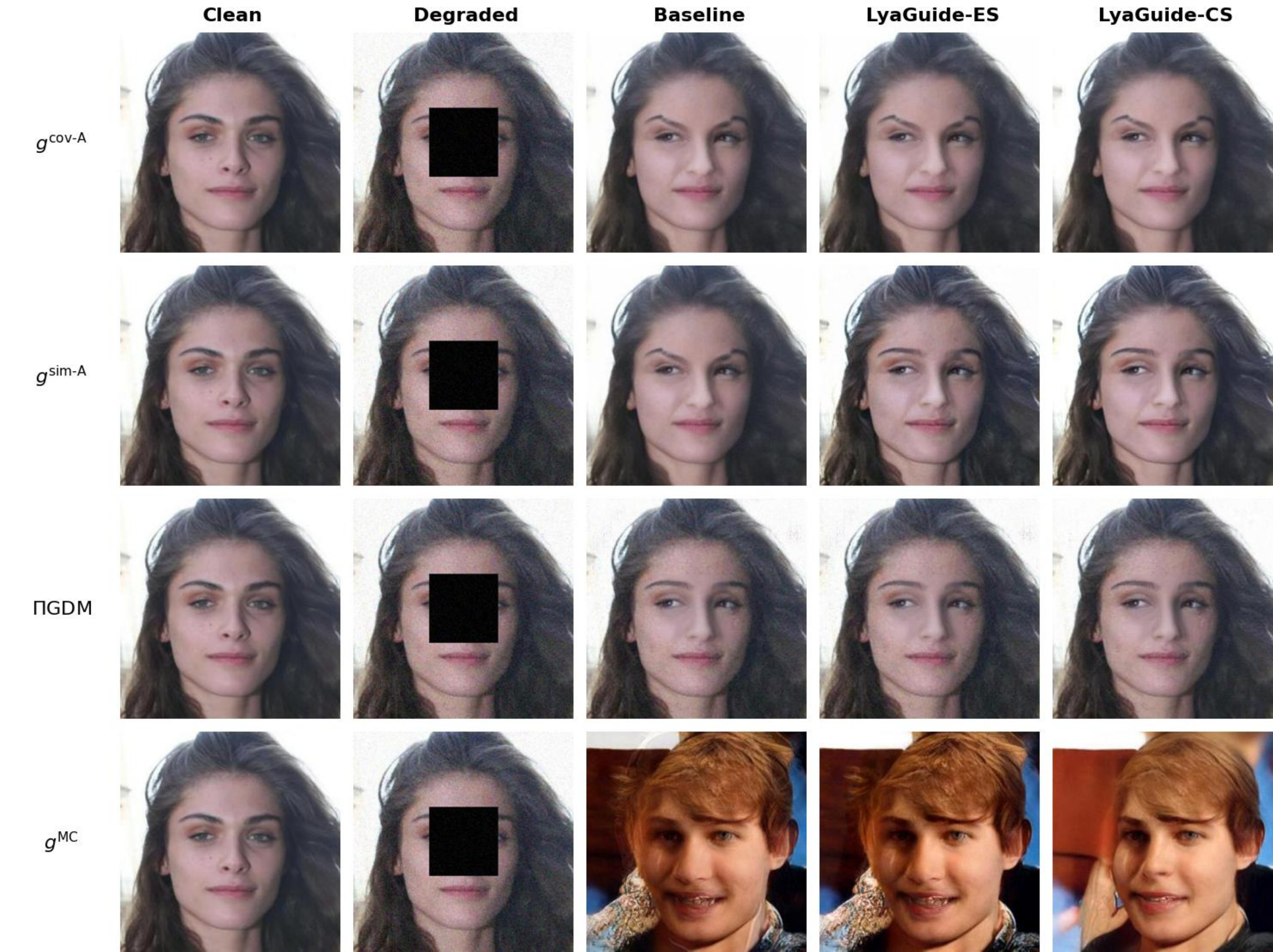}
	\caption{The visualization of the image inverse problems with the base flow matching model of conditional flow matching
 (CFM). Four rows show the results of four baselines with the corresponding LyaGuide results in box-inpainting task.
     }
% \vskip -0.25in
\end{figure*}

\clearpage
\newpage

\bibliographystyle{IEEEtran}
\bibliography{main}

@article{tsukamoto2021contraction,
  title={Contraction theory for nonlinear stability analysis and learning-based control: A tutorial overview},
  author={Tsukamoto, Hiroyasu and Chung, Soon-Jo and Slotine, Jean-Jaques E},
  journal={Annual Reviews in Control},
  volume={52},
  pages={135--169},
  year={2021},
  publisher={Elsevier}
}

@book{wiener2019cybernetics,
	title={Cybernetics or Control and Communication in the Animal and the Machine},
	author={Wiener, Norbert},
	year={2019},
	publisher={MIT press}
}

@book{mao2007stochastic,
	title={Stochastic differential equations and applications},
	author={Mao, Xuerong},
	year={2007},
	publisher={Elsevier}
}

@article{khalil2002nonlinear,
  title={Nonlinear systems third edition},
  author={Khalil, Hassan K},
  journal={Patience Hall},
  volume={115},
  year={2002}
}

@book{parrilo2000structured,
  title={Structured Semidefinite Programs and Semialgebraic Geometry Methods in Robustness and Optimization},
  author={Parrilo, Pablo A},
  year={2000},
  publisher={California Institute of Technology}
}

@inproceedings{chang2020neural,
  title={Neural lyapunov control},
  author={Chang, Ya-Chien and Roohi, Nima and Gao, Sicun},
  booktitle={Proceedings of the 33rd International Conference on Neural Information Processing Systems},
  pages={3245--3254},
  year={2019}
}

@inproceedings{
	zhang2022neural,
	title={Neural Stochastic Control},
	author={Jingdong Zhang and Qunxi Zhu and Wei Lin},
	booktitle={Advances in Neural Information Processing Systems},
	year={2022}
}

@inproceedings{qin2020learning,
	title={Learning Safe Multi-agent Control with Decentralized Neural Barrier Certificates},
	author={Qin, Zengyi and Zhang, Kaiqing and Chen, Yuxiao and Chen, Jingkai and Fan, Chuchu},
	booktitle={International Conference on Learning Representations},
	year={2020}
}

@inproceedings{sun2021learning,
	title={Learning certified control using contraction metric},
	author={Sun, Dawei and Jha, Susmit and Fan, Chuchu},
	booktitle={Conference on Robot Learning},
	pages={1519--1539},
	year={2021},
	organization={PMLR}
}

@article{dawson2023safe,
	title={Safe control with learned certificates: A survey of neural lyapunov, barrier, and contraction methods for robotics and control},
	author={Dawson, Charles and Gao, Sicun and Fan, Chuchu},
	journal={IEEE Transactions on Robotics},
	year={2023},
	publisher={IEEE}
}

@article{chow2019lyapunov,
	title={Lyapunov-based safe policy optimization for continuous control},
	author={Chow, Yinlam and Nachum, Ofir and Faust, Aleksandra and Duenez-Guzman, Edgar and Ghavamzadeh, Mohammad},
	journal={arXiv preprint arXiv:1901.10031},
	year={2019}
}

@inproceedings{zhang2022sync,
	title={SYNC: Safety-aware neural control for stabilizing stochastic delay-differential equations},
	author={Zhang, Jingdong and Zhu, Qunxi and Yang, Wei and Lin, Wei},
	booktitle={The Eleventh International Conference on Learning Representations},
	year={2022}
}

@inproceedings{finn2017model,
  title={Model-Agnostic Meta-Learning for Fast Adaptation of Deep Networks},
  author={Finn, Chelsea and Abbeel, Pieter and Levine, Sergey},
  booktitle={ICML},
  year={2017}
}

@inproceedings{rusu2018meta,
  title={Meta-learning with latent embedding optimization},
  author={Rusu, Andrei A and Rao, Dushyant and Sygnowski, Jakub and Pascanu, Razvan and Osindero, Simon and Hadsell, Raia},
  booktitle={ICLR},
  year={2019}
}

@inproceedings{rebuffi2017learning,
  title={Learning multiple visual domains with residual adapters},
  author={Rebuffi, Sylvestre-Alvise and Bilen, Hakan and Vedaldi, Andrea},
  booktitle={NeurIPS},
  year={2017}
}

@inproceedings{lipman2023flow,
  title={Flow Matching for Generative Modeling},
  author={Lipman, Yaron and Chen, Ricky TQ and Ben-Hamu, Heli and Nickel, Maximilian and Le, Matt},
  booktitle={11th International Conference on Learning Representations, ICLR 2023},
  year={2023}
}

@article{dhariwal2021diffusion,
  title={Diffusion models beat gans on image synthesis},
  author={Dhariwal, Prafulla and Nichol, Alexander},
  journal={Advances in neural information processing systems},
  volume={34},
  pages={8780--8794},
  year={2021}
}

@article{zhang2024generalized,
  title={Generalized protein pocket generation with prior-informed flow matching},
  author={Zhang, Zaixi and Zitnik, Marinka and Liu, Qi},
  journal={Advances in neural information processing systems},
  volume={37},
  pages={38559--38589},
  year={2024}
}

@inproceedings{ho2021classifier,
  title={Classifier-Free Diffusion Guidance},
  author={Ho, Jonathan and Salimans, Tim},
  booktitle={NeurIPS 2021 Workshop on Deep Generative Models and Downstream Applications}
}

@article{polyakov2012nonlinear,
  title={Nonlinear feedback design for fixed-time stabilization of linear control systems},
  author={Polyakov, Andrey},
  journal={IEEE Transactions on Automatic Control},
  volume={57},
  number={8},
  pages={2106--2110},
  year={2012},
  publisher={IEEE}
}

@article{polyakov2015finite,
  title={Finite-time and fixed-time stabilization: Implicit Lyapunov function approach},
  author={Polyakov, Andrey and Efimov, Denis and Perruquetti, Wilfrid},
  journal={Automatica},
  volume={51},
  pages={332--340},
  year={2015},
  publisher={Elsevier}
}

@article{song2019generative,
  title={Generative modeling by estimating gradients of the data distribution},
  author={Song, Yang and Ermon, Stefano},
  journal={Advances in neural information processing systems},
  volume={32},
  year={2019}
}

@article{deenen2021projection,
    author = {Deenen, Daniel Andreas and Sharif, Bardia and van den Eijnden, Sebastiaan and Nijmeijer, Hendrik and Heemels, Maurice and Heertjes, Marcel},
    title = {Projection-based integrators for improved motion control: {F}ormalization, well-posedness and stability of hybrid integrator-gain systems},
    journal = {{A}utomatica},
    volume = {133},
    pages = {109830},
    year = {2021},
    publisher = {Elsevier}
}

@article{lecun2006ebm,
  title={A tutorial on energy-based learning},
  author={LeCun, Yann and Chopra, Sumit and Hadsell, Raia and Ranzato, M and Huang, Fujie and others},
  journal={Predicting structured data},
  volume={1},
  number={0},
  year={2006}
}

@inproceedings{janner2022planning,
  title={Planning with Diffusion for Flexible Behavior Synthesis},
  author={Janner, Michael and Du, Yilun and Tenenbaum, Joshua and Levine, Sergey},
  booktitle={International Conference on Machine Learning},
  pages={9902--9915},
  year={2022},
  organization={PMLR}
}

@inproceedings{fengguidance,
  title={On the Guidance of Flow Matching},
  author={Feng, Ruiqi and Yu, Chenglei and Deng, Wenhao and Hu, Peiyan and Wu, Tailin},
  booktitle={Forty-second International Conference on Machine Learning},
year={2025}
}

@article{zhang2023towards,
  title={Towards controllable diffusion models via reward-guided exploration},
  author={Zhang, Hengtong and Xu, Tingyang},
  journal={arXiv preprint arXiv:2304.07132},
  year={2023}
}

@inproceedings{song2021score,
  title={Score-based generative modeling through stochastic differential equations},
  author={Song, Yang and Sohl-Dickstein, Jascha and Kingma, Diederik P. and Kumar, Abhishek and Ermon, Stefano and Poole, Ben},
  booktitle={International Conference on Learning Representations},
  year={2021}
}

@article{du2019implicit,
  title={Implicit generation and modeling with energy based models},
  author={Du, Yilun and Mordatch, Igor},
  journal={Advances in neural information processing systems},
  volume={32},
  year={2019}
}

@inproceedings{lee2025calibrated,
  title={Calibrated multi-preference optimization for aligning diffusion models},
  author={Lee, Kyungmin and Li, Xiahong and Wang, Qifei and He, Junfeng and Ke, Junjie and Yang, Ming-Hsuan and Essa, Irfan and Shin, Jinwoo and Yang, Feng and Li, Yinxiao},
  booktitle={Proceedings of the Computer Vision and Pattern Recognition Conference},
  pages={18465--18475},
  year={2025}
}

@article{artstein1983stabilization,
  title={Stabilization with relaxed controls},
  author={Artstein, Zvi},
  journal={Nonlinear Analysis: Theory, Methods \& Applications},
  volume={7},
  number={11},
  pages={1163--1173},
  year={1983},
  publisher={Elsevier}
}

@article{sontag1989universal,
  title={A ‘universal’construction of Artstein's theorem on nonlinear stabilization},
  author={Sontag, Eduardo D},
  journal={Systems \& control letters},
  volume={13},
  number={2},
  pages={117--123},
  year={1989},
  publisher={Elsevier}
}

@article{yang2025neural,
  title={Neural Event-Triggered Control with Optimal Scheduling},
  author={Yang, Luan and Zhang, Jingdong and Zhu, Qunxi and Lin, Wei},
  journal={arXiv preprint arXiv:2507.14653},
  year={2025}
}

@inproceedings{liuflow,
  title={Flow Straight and Fast: Learning to Generate and Transfer Data with Rectified Flow},
  author={Liu, Xingchao and Gong, Chengyue and others},
  booktitle={The Eleventh International Conference on Learning Representations},
year={2023}
}

@article{tong2023conditional,
  title={Conditional flow matching: Simulation-free dynamic optimal transport},
  author={Tong, Alexander and Malkin, Nikolay and Huguet, Guillaume and Zhang, Yanlei and Rector-Brooks, Jarrid and Fatras, Kilian and Wolf, Guy and Bengio, Yoshua},
  journal={arXiv preprint arXiv:2302.00482},
  volume={2},
  number={3},
  year={2023}
}

@inproceedings{blacktraining,
  title={Training Diffusion Models with Reinforcement Learning},
  author={Black, Kevin and Janner, Michael and Du, Yilun and Kostrikov, Ilya and Levine, Sergey},
  booktitle={The Twelfth International Conference on Learning Representations},
year={2023}
}

@inproceedings{lu2023contrastive,
  title={Contrastive energy prediction for exact energy-guided diffusion sampling in offline reinforcement learning},
  author={Lu, Cheng and Chen, Huayu and Chen, Jianfei and Su, Hang and Li, Chongxuan and Zhu, Jun},
  booktitle={International Conference on Machine Learning},
  pages={22825--22855},
  year={2023},
  organization={PMLR}
}

@article{fan2025cfg,
  title={CFG-Zero*: Improved Classifier-Free Guidance for Flow Matching Models},
  author={Fan, Weichen and Zheng, Amber Yijia and Yeh, Raymond A and Liu, Ziwei},
  journal={CoRR},
  year={2025}
}

@inproceedings{zhang2024fessnc,
  title={FESSNC: Fast Exponentially Stable and Safe Neural Controller},
  author={Zhang, Jingdong and Yang, Luan and Zhu, Qunxi and Lin, Wei},
  booktitle={International Conference on Machine Learning},
  pages={60076--60098},
  year={2024},
  organization={PMLR}
}

@inproceedings{song2023pseudoinverse,
  title={Pseudoinverse-guided diffusion models for inverse problems},
  author={Song, Jiaming and Vahdat, Arash and Mardani, Morteza and Kautz, Jan},
  booktitle={International Conference on Learning Representations},
  year={2023}
}

@inproceedings{sprague2024stability,
  title={Incorporating stability into flow matching},
  author={Sprague, Christopher Iliffe and Elofsson, Arne and Azizpour, Hossein},
  booktitle={ICML 2024 Workshop on Structured Probabilistic Inference $\{$$\backslash$\&$\}$ Generative Modeling},
  year={2024}
}

@article{kang2021stable,
  title={Stable neural ode with lyapunov-stable equilibrium points for defending against adversarial attacks},
  author={Kang, Qiyu and Song, Yang and Ding, Qinxu and Tay, Wee Peng},
  journal={Advances in Neural Information Processing Systems},
  volume={34},
  pages={14925--14937},
  year={2021}
}

@article{liu2023pinn,
  title={Physics-Informed Neural Network Lyapunov Functions: PDE Characterization, Learning, and Verification},
  author={Liu, J. and Meng, Y. and Fitzsimmons, M.},
  journal={Journal of Machine Learning Research},
  year={2023},
  volume={24},
  number={210},
  pages={1--36}
}

@inproceedings{albergo2023stochastic,
  title={Stochastic Interpolant: A New Framework for Generative Modeling},
  author={Albergo, M. S. and Rizzi, M. and Cranmer, K.},
  booktitle={International Conference on Machine Learning (ICML)},
  year={2023}
}

@inproceedings{chen2023phase,
  title={Generative Modeling with Phase Stochastic Bridges},
  author={Chen, T. and Gu, J. and Dinh, L.},
  booktitle={International Conference on Learning Representations (ICLR)},
      year={2023}}

@inproceedings{
balcerak2025energy,
title={Energy Matching: Unifying Flow Matching and Energy-Based Models for Generative Modeling},
author={Michal Balcerak and Tamaz Amiranashvili and Antonio Terpin and Suprosanna Shit and Lea Bogensperger and Sebastian Kaltenbach and Petros Koumoutsakos and Bjoern Menze},
booktitle={The Thirty-ninth Annual Conference on Neural Information Processing Systems},
year={2025},
}

@article{song2021solving,
  title={Solving inverse problems in medical imaging with score-based generative models},
  author={Song, Yang and Shen, Liyue and Xing, Lei and Ermon, Stefano},
  journal={arXiv preprint arXiv:2111.08005},
  year={2021}
}

@inproceedings{chungdiffusion,
  title={Diffusion Posterior Sampling for General Noisy Inverse Problems},
  author={Chung, Hyungjin and Kim, Jeongsol and Mccann, Michael Thompson and Klasky, Marc Louis and Ye, Jong Chul},
  booktitle={The Eleventh International Conference on Learning Representations}
}

@article{kawar2022denoising,
  title={Denoising diffusion restoration models},
  author={Kawar, Bahjat and Elad, Michael and Ermon, Stefano and Song, Jiaming},
  journal={Advances in neural information processing systems},
  volume={35},
  pages={23593--23606},
  year={2022}
}

@inproceedings{wangzero,
  title={Zero-Shot Image Restoration Using Denoising Diffusion Null-Space Model},
  author={Wang, Yinhuai and Yu, Jiwen and Zhang, Jian},
  booktitle={The Eleventh International Conference on Learning Representations}
}

@article{pokle2023training,
  title={Training-free linear image inverses via flows},
  author={Pokle, Ashwini and Muckley, Matthew J and Chen, Ricky TQ and Karrer, Brian},
  journal={arXiv preprint arXiv:2310.04432},
  year={2023}
}

@article{zhang2024flow,
  title={Flow priors for linear inverse problems via iterative corrupted trajectory matching},
  author={Zhang, Yasi and Yu, Peiyu and Zhu, Yaxuan and Chang, Yingshan and Gao, Feng and Wu, Ying N and Leong, Oscar},
  journal={Advances in Neural Information Processing Systems},
  volume={37},
  pages={57389--57417},
  year={2024}
}

@inproceedings{martin2025pnp,
  title={PNP-FLOW: Plug-And-Play Image Restoration with Flow Matching},
  author={Martin, S{\'e}gol{\`e}ne and Gagneux, Anne and Hagemann, Paul and Steidl, Gabriele},
  booktitle={International Conference on Learning Representations},
  year={2025}
}

@inproceedings{yan2025fig,
  title={Fig: Flow with interpolant guidance for linear inverse problems},
  author={Yan, Yici and Zhang, Yichi and Meng, Xiangming and Zhao, Zhizhen},
  booktitle={The Thirteenth International Conference on Learning Representations},
  year={2025}
}

@inproceedings{houlsby2019parameter,
  title={Parameter-efficient transfer learning for NLP},
  author={Houlsby, Neil and Giurgiu, Andrei and Jastrzebski, Stanislaw and Morrone, Bruna and De Laroussilhe, Quentin and Gesmundo, Andrea and Attariyan, Mona and Gelly, Sylvain},
  booktitle={International conference on machine learning},
  pages={2790--2799},
  year={2019},
  organization={PMLR}
}

@inproceedings{rebuffi2018efficient,
  title={Efficient parametrization of multi-domain deep neural networks},
  author={Rebuffi, Sylvestre-Alvise and Bilen, Hakan and Vedaldi, Andrea},
  booktitle={Proceedings of the IEEE conference on computer vision and pattern recognition},
  pages={8119--8127},
  year={2018}
}

@inproceedings{perez2018film,
  title={Film: Visual reasoning with a general conditioning layer},
  author={Perez, Ethan and Strub, Florian and De Vries, Harm and Dumoulin, Vincent and Courville, Aaron},
  booktitle={Proceedings of the AAAI conference on artificial intelligence},
  volume={32},
  number={1},
  year={2018}
}

@inproceedings{chen2024flow,
  title={Flow Matching on General Geometries},
  author={Chen, Ricky TQ and Lipman, Yaron},
  booktitle={12th International Conference on Learning Representations, ICLR 2024},
  year={2024}
}

@article{fu2020d4rl,
  title={D4rl: Datasets for deep data-driven reinforcement learning},
  author={Fu, Justin and Kumar, Aviral and Nachum, Ofir and Tucker, George and Levine, Sergey},
  journal={arXiv preprint arXiv:2004.07219},
  year={2020}
}

@article{levine2018reinforcement,
  title={Reinforcement learning and control as probabilistic inference: Tutorial and review},
  author={Levine, Sergey},
  journal={arXiv preprint arXiv:1805.00909},
  year={2018}
}

\end{document}